\def\isarxiv{1}
\ifdefined\isarxiv
\documentclass[11pt]{article}

\usepackage[numbers]{natbib}

\else

\fi

\usepackage{amsmath}
\usepackage{amsthm}
\usepackage{amssymb}
\usepackage{algorithm}
\usepackage{subfig}
\usepackage{algpseudocode}
\usepackage{graphicx}
\usepackage{grffile}
\usepackage{wrapfig,epsfig}
\usepackage{url}
\usepackage{xcolor}
\usepackage{epstopdf}

\usepackage{bbm}
\usepackage{dsfont}

\allowdisplaybreaks

\usepackage{tikz}

\usetikzlibrary{arrows}
\ifdefined\isarxiv
\usepackage[margin=1in]{geometry}
\fi
\graphicspath{{./figs/}}

\definecolor{b2}{RGB}{51,153,255}
\definecolor{mygreen}{RGB}{80,180,0}

\theoremstyle{plain}
\newtheorem{theorem}{Theorem}[section]
\newtheorem{lemma}[theorem]{Lemma}

\newtheorem{fact}[theorem]{Fact}

\newtheorem{corollary}[theorem]{Corollary}

\theoremstyle{definition}
\newtheorem{definition}[theorem]{Definition}

\newcommand{\wh}{\widehat}

\renewcommand{\epsilon}{\varepsilon}
\renewcommand{\phi}{\varphi}

\newcommand{\N}{\mathbb{N}}
\newcommand{\R}{\mathbb{R}}

\renewcommand{\hat}{\wh}

\newcommand{\E}{\mathbb{E}}
\newcommand{\Var}{\mathrm{Var}}

\makeatletter
\newcommand*{\RN}[1]{\expandafter\@slowromancap\romannumeral #1@}
\makeatother

\newcommand{\anng}{\mathsf{ANN}\text{-}\mathsf{Graph}} 

\newcommand{\vol}{\mathrm{Vol}}
\newcommand{\nns}{\mathsf{NN}\text{-}\mathsf{Search}}

\title{A Theoretical Analysis of Nearest Neighbor Search on Approximate Near Neighbor Graph}

\begin{document}

\ifdefined\isarxiv
\title{A Theoretical Analysis Of Nearest Neighbor Search On Approximate Near Neighbor Graph}

\author{
Anshumali Shrivastava\thanks{\texttt{anshumali@rice.edu}. Rice University.}
\and
Zhao Song\thanks{\texttt{zsong@adobe.com}. Adobe Research.}
\and
Zhaozhuo Xu\thanks{\texttt{zx22@rice.edu}. Rice University.}
}

\date{}

\begin{titlepage}
    \maketitle
    \begin{abstract}
        
Graph-based algorithms have demonstrated state-of-the-art performance in the nearest neighbor search ($\nns$) problem. These empirical successes urge the need for theoretical results that guarantee the search quality and efficiency of these algorithms. However, there exists a practice-to-theory gap in the graph-based $\nns$ algorithms. Current theoretical literature focuses on greedy search on exact near neighbor graph while practitioners use approximate near neighbor graph ($\anng$) to reduce the preprocessing time. This work bridges this gap by presenting the theoretical guarantees of solving $\nns$ via greedy search on $\anng$ for low dimensional and dense vectors. To build this bridge, we leverage several novel tools from computational geometry. Our results provide quantification of the trade-offs associated with the approximation while building a near neighbor graph. We hope our results will open the door for more provable efficient graph-based $\nns$ algorithms.

    \end{abstract}
    \thispagestyle{empty}
\end{titlepage}
\else
\fi

\section{Introduction}
In this paper, we study the nearest neighbor search ($\nns$) problem. Let $P$ denotes an $n$-vector, $d$-dimensional dataset. Given a query vector $q \in R^d$, the $\nns$ aims at retrieving a vector $p$ from $P$ that has the minimum distance with $q$. Here the distance could be Euclidean distance, cosine distance, or negative inner product.  $\nns$ is a fundamental problem in machine learning. $\nns$ is the building block of the well-known k-nearest neighbor algorithm~\citep{fh89,a92}, which has wide applications in computer vision~\citep{sdi06}, language processing~\citep{klj+19} and recommendation systems~\citep{skk+01}. Moreover, recent research suggests that $\nns$ algorithms could be used to scale up the training of deep neural networks~\citep{cmf+20,clp+20}. 

Graph-based algorithms have achieved state-of-the-art performance in $\nns$~\citep{has+11,fxw+17,my18,tzx+19,ztx+19,ztl20,txz+21}. These algorithms first preprocess the dataset $P$ into a directed graph $G$, where each vector in the dataset is represented by a vertex. An edge exists from vertex $p_1$ to vertex $p_2$ indicates that $p_2$ is close to $p_1$ in terms of some distance measure. In this situation, $p_2$ also belongs to the out-neighbors set of $p_1$.   Next, given a query vector $q$, starting from an arbitrary vertex $p$, the graph-based $\nns$ algorithms perform a greedy search on the graph: find the vertex $p'$ from the out-neighbors of $p$ that has the minimum distance with $q$, if $p'$ is closer to $q$ than $p$, set $p'$ to be the new $p$ and repeat this process. The major intuition for this greedy search is the six degrees of separation~\citep{nb06} that any two people in the world could be connected with a maximum of six friends of a friend steps. Following this idea, graph-based algorithms could potentially reduce the complexity of $\nns$ to sub-linear in the number of vectors in the dataset. We should also note that preprocessing the dataset $P$ is required in graph-based $\nns$ algorithms. However, this preprocessing is tolerable in real-world scenarios. For instance, in recommendation systems~\citep{fgz+19,ah21}, $P$ is the set of item embeddings, and the query $q$ represents a user embedding. Due to the massive amount of users and their frequent queries, the cumulative time savings introduced by graph-based $\nns$ algorithms would exceed the 
cost in preprocessing.

The practical success of graph-based $\nns$ algorithms urges the development of theory. While other $\nns$ algorithms such as hashing~\citep{ar15,alrw17} and quantization~\citep{jds10} are associated with well-established guarantees in space-time trade-offs, the theory of graph-based $\nns$ algorithms is still in an early stage. One major practice-to-theory gap is that: recent theoretical analysis~\citep{l18,ps20} focuses on providing space and time guarantees for greedy search on the exact near neighbor graph. In this graph $G$ for an $n$-vector dataset $P$, the out-neighbors of any vertex $p\in P$ contains all the datapoints that have distances smaller than a threshold with $p$. However, in many applications, practitioners reduce the preprocessing time or space by approximating the exact near neighbor graph~\citep{cfs09,dml11,zhg+13,fxw+17,my18,bbm18,ztx+19,ztl20,txz+21}. As a result, the approximate near neighbor graph ($\anng$) has the following propriety: for any vertex $p\in P$, every vertex $p'\in P$ with distance smaller than the threshold is not guaranteed to be $p$'s out-neighbors. It remains unknown how to provide the trade-offs for greedy search on the $\anng$.

\subsection{Our Contributions}
In this paper, we fill this practice-to-theory gap by providing the trade-offs for greedy style $\nns$ algorithms on $\anng$. We take $\nns$ with angular distance as an example and present our main results as below:

\begin{theorem}[An informal version of Theorem~\ref{thm:main:formal}]\label{thm:main:informal}
Let $\alpha\in (0,1)$ and $\delta\in (0,1)$. Let $P$ denote a $n$-vector dataset on a unit sphere $\mathbb{S}^{d-1}$. Let directed graph $G$ denote an exact near neighbor graph where each vertex is a vector $p\in P$. Moreover, an edge from $p_1$ to $p_2$ exists on $G$ if and only if $\langle p_1,p_2 \rangle \geq \alpha$. Let ${\cal A}$ denote a greedy search algorithm on $G$ that solves the nearest neighbor search problem. 

If we approximate $G$ by graph $\hat{G}$ such that any edge exists on $G$ would also exist on $\hat{G}$ with probability at least $\delta$, then we show that, compared to performing ${\cal A}$ on $G$,  performing ${\cal A}$ on $\hat{G}$ solves the nearest neighbor search problem using same query time and space but the failure probability would be raised to the power of $\Theta(\delta)$.
\end{theorem}

 Our theorem provides the trade-off between query time and failure probability when we approximate the near neighbor graph with a probability factor $\delta$.
 Specifically, as shown in Figure~\ref{fig:main}, our theorem contains three cases of $\anng$:
 \begin{itemize}
     \item $\anng$ with Adaptive Sampling (see Figure~\ref{fig:main}, middle left): for any vertex $p\in P$, its neighbors $q$ be sampled with probability monotonically decrease over the angle between $p$ and $q$
     \item Randomized $\anng$ (see Figure~\ref{fig:main}, middle right): Any edge exists on $G$ would also exist on $\hat{G}$ with probability $\delta$. This case corresponds to  where a vertex $p\in P$ compute with a uniform sampled subset of $P$ and determine the out-neighbors
     \item Randomized $\anng$ with random edges~\cite{jkl+20} (see Figure~\ref{fig:main}, right): we add random out-neighbors to the vertex on $\hat{G}$ in case (2) to improve the connectivity of the graph
 \end{itemize}

 From the theorem, we observe that as $\delta$ decreases, we would maintain query time but a larger failure probability. This is because the number of true out-neighbors on the graph decreases as $\delta$ decreases. As a result, we would have a larger probability of failure due to the absence of important edges.

\begin{figure*}[t]
  \centering
    \includegraphics[width=0.7\linewidth]{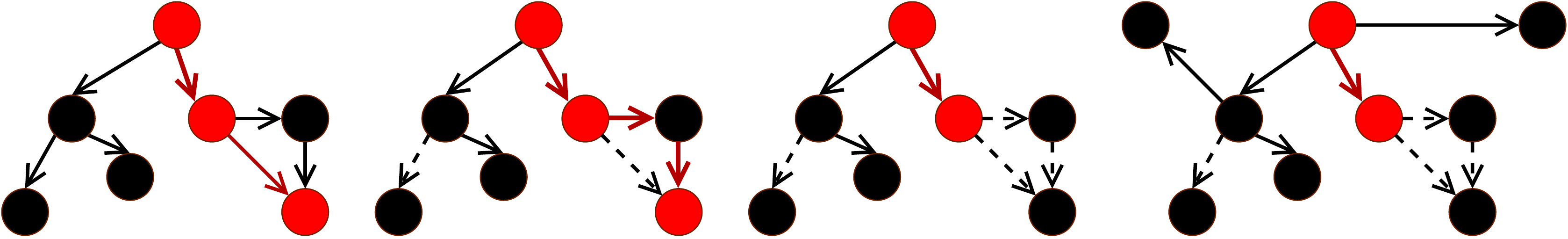}
    \caption{Illustration of our results. The black arrow represents the edge on the graph. The dotted black arrow represents the edges that disappears on the $\anng$. The red arrow denote the greedy search direction. Left: greedy search on original graph. Middle left: if index by adaptive sampling, closer node have higher probability to be connected on graph. Therefore, we need to pay more search steps in query. Middle right: random sampling of nodes when indexing the graph cause the absence of the edges and higher probability. Right: add random edges on the graph will increase the search complexity.}\label{fig:main}
\end{figure*} 
\section{Related Work}
Graph-based algorithms~\citep{fxw+17,my18,ztx+19,ztl20,txz+21} have achieved state-of-the-art performance in $\nns$ benchmarks~\citep{abf17}. Compared to hashing~\citep{dii+04}, quantization~\citep{jds10} and tree-based approaches~\citep{rp92}, graph-based $\nns$ algorithms significantly improve the search efficiency when we would like to obtain high recall of the nearest neighbor.  The empirical results indicate that greedy search on the near neighbor graph reduces the time complexity to reach the nearest neighbor. 

The theoretical foundation of graph-based $\nns$ algorithms is a challenging topic. \cite{l18} takes a step towards an asymptotic analysis of greedy search on the exact near neighbor graph in high dimensional settings. The proof in \citep{l18} indicates that graph-based $\nns$ algorithms achieves comparable performance with the optimal hashing-based $\nns$ algorithms~\citep{ar15,alrw17} in terms of space-time trade-offs. \cite{ps20} extends the theoretical analysis in \citep{l18}  to the dense vectors.
\cite{ps20} also discusses several modifications in the greedy search algorithm on the graph. However, these related works restrict the discussion to the exact near neighbor graph.
As the $\anng$ is widely used in practice to reduce the space complexity of storing the graph, the current theoretical analysis does not provide insights for the development of new graph-based $\nns$ algorithms.

\section{Preliminary}
In this section, we introduce the preliminary knowledge for our analysis on the $\anng$ based $\nns$ algorithm. We start by presenting the basic notations for this paper. Next, we introduce our definition of $\nns$, including the assumptions over the relationship between dataset size and dimensionality. Finally, we introduce the tools from geometry that we use in this paper.

\subsection{Notations}
We use $\Pr[]$ and $\E[]$ to denote probability and expectation. We use $\Var[]$ to denote the variance. We use $\max\{a,b\}$ to denote the maximum between $a$ and $b$. We use $\min \{a,b\}$ (resp. $\max\{a,b\}$) to denote the minimum (reps. maximum) between $a$ and $b$. For a vector $x$, we use $\| x \|_2 := ( \sum_{i=1}^n x_i^2 )^{1/2}$ to denote its $\ell_2$ norm. We use ${\cal U}(X)$ to denote the uniform distribution over set $X\subset \R^d$.  Let $B(n,p)$ denote a Binomial distribution with $n$ independent Bernoulli trials and each trial has success probability $p$. For two vector $x,y$, we use $\theta_{x,y}\in [0,\pi]$ to denote their angle. We use $\sin \theta_{x,y}$ to denote the sine function for angle $\theta_{x,y}$.  We use  $\mathbb{S}^{d-1}$  to denote a unit sphere in $\R^d$. For $x,y \in \R$, we denote $x \lesssim y$ if $x\leq Cy$ for some constant $C>0$. In this paper, we measure the time and space complexity in real RAM.

\subsection{Nearest Neighbor Search in Dense Regime}
In our work, we study the nearest neighbor search in a dense regime. We start with defining the dense dataset with parameter $\omega$. This definition is standard in the analysis of graph-based $\nns$ algorithms~\citep{l18,ps20}.

\begin{definition}[Dense dataset]\label{def:dense_dataset}
Let $\omega>1$ denote a fixed parameter. We say $n$-point dataset in $\mathcal{S}^{d-1}$ is $\omega$-dense if $\omega=(\log n)/d$.
\end{definition}

Next, we present a fact for the $\omega$-dense dataset.
\begin{fact}\label{fact:omega_to_d}
For an $n$-point, $\omega$-dense dataset in $\mathcal{S}^{d-1}$, we show that $2^{-d\omega}=\frac{1}{n}$.
\end{fact}

Fact~\ref{fact:omega_to_d} is used in proving the main results of the paper. Besides, we also define a function that helps the proof.

\begin{definition}[$\alpha$ function]\label{def:alpha_function}
For a fixed parameter $\omega>1$, we define function $\alpha_x:[1, \log \omega] \to [0,1] $ as follows
\begin{align*}
    \alpha_x := \sqrt{1 - x^2 \cdot 2^{-2\omega}}.
\end{align*}
\end{definition}

Next, we introduce our formulation of $\nns$ problem in Definition~\ref{def:nn:formal}.

\begin{definition}[Nearest  Neighbor Search ($\nns$)]\label{def:nn:formal}
Let $\omega>1$ and $r\in (1,2^{\omega})$. Let $P\subset\mathbb{S}^{d-1}$ denote an $n$-vector, $\omega$-dense (Definition~\ref{def:dense_dataset}) dataset where every $p\in P$ is i.i.d. sampled from ${\cal U}(\mathbb{S}^{d-1})$.
The goal of the $r$-Nearest Neighbor ($\nns$) search is to construct a data structure that for any query $q \in \mathbb{S}^{d-1}$ with the premise that $\min_{p\in P} \sin \theta_{p,q} \leq r\cdot 2^{-\omega}$,
return a vector $p' \in P$ such that $\sin  \theta_{p',q}  \leq r\cdot 2^{-\omega}$.
\end{definition}

\paragraph{All Arguments Directly Applies to the Euclidean Distance:} In this paper, we set the distance measure for $\nns$ as the sine function. As the dataset and queries are in the unit sphere, Definition~\ref{def:nn:formal} naturally relates to $\nns$ in Euclidean distance.

\begin{figure}[t]
  \centering
\setlength{\abovecaptionskip}{0pt}
 \setlength{\belowcaptionskip}{-5pt}
    \includegraphics[width=0.5\linewidth]{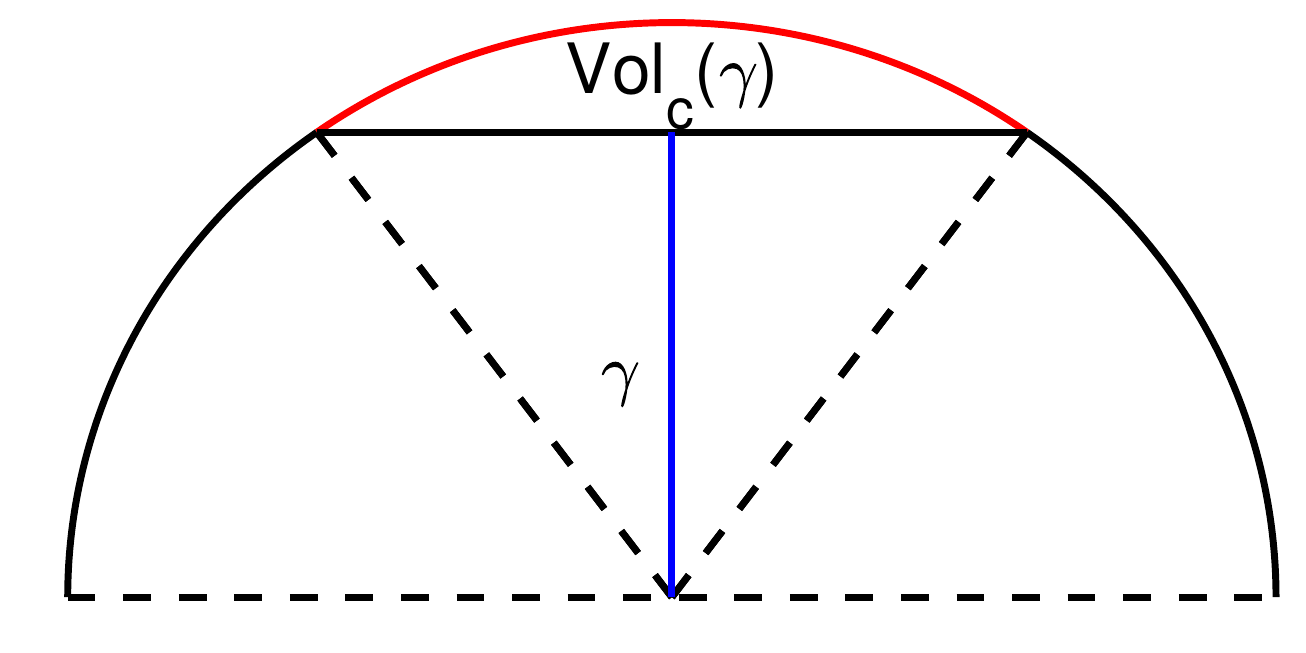}%
    \caption{\footnotesize Example of spherical cap and its volume.}\label{fig:volc}
\end{figure}

\subsection{Tools}
In this section, we introduce the tools from geometry and probability that helps the proof our theoretical results in this paper.
We start with the definitions of spherical caps and their volumes. Next, we introduce the definition of wedges and their volumes. Next, we introduce the Chebyshev's inequality and a corollary based on it.
\subsubsection{Volumes of Spherical Caps}
We start with introducing the definition of spherical caps on a unit sphere. 

\begin{definition}[Spherical cap]\label{def:spherical_cap}
Given a vector $x\in\mathbb{S}^{d-1}$, we define $C_x(\gamma)=\{y \in \mathbb{S}^{d-1}: \langle x, y \rangle \geq \gamma\}$ as a spherical cap of height $\gamma\in [0,1]$ centered at $x$. 
\end{definition}

Next, we define the notations for the volume of a spherical cap. We also use  Figure~\ref{fig:volc} illustrate it in 2-dimensional.

\begin{definition}[Volume of spherical cap]\label{def:vol_spherical_cap}
For a spherical cap $C_x(\gamma)$ defined in Definition~\ref{def:spherical_cap}, we denote $\vol_{c}(x,\gamma)$ as its volume relative to the volume of $\mathbb{S}^{d-1}$. However, as the volume of spherical cap does not depend on the center $x$, we write $\vol_{c}(x,\gamma)$ as $\vol_c(\gamma)$ for simplicity.
\end{definition}

Next, we list out a statement from~\citep{ps20} that lower bounds the volume of a spherical cap in Definition~\ref{def:spherical_cap}.

\begin{lemma}[\cite{ps20}]\label{lem:vol_cap}
Let $d\in \N_{+}$. Let $\gamma\in [0,1]$ denote a parameter. Let $\vol_c(\gamma)$ (see Definition~\ref{def:vol_spherical_cap}) denote the volume of a spherical cap $C_x(\gamma)$ (see Definition~\ref{def:spherical_cap}) defined on the unit sphere $\mathbb{S}^{d-1}$. Then we have
\begin{align*}
    \vol_c(\gamma) \gtrsim d^{-1/2}   (\sqrt{1 - \gamma^2})^{d} .
\end{align*}
\end{lemma} 

\subsection{Volume of Wedges}
We start with introducing the definition of wedge, which is the overlap area between two spherical caps.
\begin{definition}[Wedge]\label{def:wedge}
Given two spherical caps $C_x(\beta)$ and $C_y(\gamma)$ defined in Definition~\ref{def:spherical_cap}, we define the wedge $W_{x,y}$ as $W_{x,y}(\beta,\gamma)=C_x(\beta)\cap C_y(\gamma)$. 
\end{definition}

Next, we define the volume of wedge $W_{x,y}$ in Definition~\ref{def:wedge} as 
\begin{definition}[Volume of wedge]\label{def:volume_wedge}
Let $\theta = \arccos \langle x, y \rangle$. We use $\vol_w(\beta,\gamma, \theta)$ to denote the volume of $W_{x,y}(\beta,\gamma)$ (see Definition~\ref{def:wedge}) relative to the volume of $\mathbb{S}^{d-1}$. Note that the volume of wedge does not depend on the center $x$ and $y$, and only depends on $\beta,\gamma$ and $\theta$. 
\end{definition}

Next, we list out a statement from~\citep{ps20} that lower bounds the volume of wedge $W_{x,y}(\beta,\gamma)$.

\begin{lemma}[Propriety of the volume of wedge~\citep{ps20}]\label{lem:prep_vol_wedge}
Let $\theta=\arccos \langle x,y \rangle$. Let $\vol(\beta,\gamma,\theta)$ denote the volume (see Definition~\ref{def:volume_wedge}) of wedge $W_{x,y}(\beta,\gamma)$ (see Definition~\ref{def:wedge}). 
If $\beta \le \gamma \cos \theta$, then we have $ \vol_w(\beta,\gamma,\theta)>\vol_c(\gamma)/2$.
\end{lemma}

\subsection{Chebyshev's Inequality}

We use the Chebyshev's inequality for the proof of our paper.
\begin{lemma}[Chebyshev's inequality]\label{lem:chebyshev}
Let $x$ denote a random variable with expected value $\E[x]$ and variance $\Var[x]$, we show that for any $b>0$:
$\Pr[|x-\E[x]|\geq b]\leq {\Var[x]}/{b^2}$.
\end{lemma}

\begin{corollary}\label{coro:chebyshev_extend}
Let $a,b\in (0,1/2)$ denote to parameters and $a<b$. 
Let $x$ denote a random variable that follows the binomial distribution $B(n,p)$, where $p\in [a,b]$. We show that $\Pr[|x-np|\geq nb/2]\leq \frac{4}{nb}$.
\end{corollary}

\begin{proof}
We show that 
\begin{align*}
    \Pr[|x-np|\geq nb/2]
    \leq & ~\frac{4p(1-p)}{nb^2}<\frac{4}{nb},
\end{align*}
where the first step follows the Chebyshev's inequality~\ref{lem:chebyshev} and the definition of the variance of the Binomial distribution, the second step follows that $p(1-p)$ is monotonic increasing in $[a,b]$ when $b<1/2$.
\end{proof}

\section{Approximate Near Neighbor Graph}
In this section, we introduce our formulation of the approximate near neighbor graph ($\anng$).  We start with defining an oracle that is built based on the $\anng$. Next, we show the algorithm using this oracle for $\nns$. Finally, we provide the time and space complexities for the algorithm.

\subsection{Definitions}
We start by defining an oracle that is built on top of an $\anng$. The definition blocks requires statements from the supplementary material. 

\begin{definition}[Oracle]\label{def:graph_oracle}
Let $\tau>1$, $\omega>1$ and  $\delta\in (0,1)$ denote three fixed parameters. Let function $\alpha$ with parameter $\omega$ be defined as in Definition~\ref{def:alpha_function}. Let $P \subset \mathbb{S}^{d-1}$ denote a $n$-point, $\omega$-dense (Definition~\ref{def:dense_dataset}) dataset  where every $p\in P$ is i.i.d. sampled from ${\cal U}(\mathbb{S}^{d-1})$. Let $\{c_{i,j}|\ i,j\in[n]\ \textit{and} \ i\neq j  \}$ denote a set of $n(n-1)$ biased coins. When we toss any coin $c_{i,j}$ in this set where $i,j\in[n]$ and $i\neq j$, it shows head with probability at least $\delta$. 

We define the $\anng$ $G_{\tau,\delta, P}$ as a directed graph over $P$. $G_{\tau,\delta, P}$ has $n$ vertices, where each vertex represents a datapoint $p\in P$. There are no self-loops on $G_{\tau,\delta, P}$. For $i,j\in[n]$ and $i\neq j$, an edge from vertex $p_i$ to $p_j$ exists on $G_{\tau,\delta, P}$ if and only if: (1) $\langle p_i,p_j \rangle \geq \alpha_{\tau}$, (2) we flip the coin $c_{i,j}$ for one time and obtain head. 

For any $p\in P$, we define a neighbor set ${\cal N}_{\tau,\delta, P}(p)\subset P$ as: for every $p'\in {\cal N}_{\tau,\delta, P}(p)$, there exists an edge from $p$ to $p'$ on $G_{\tau,\delta, P}$.

We define an oracle ${\cal O}_{\tau,\delta, P}$ as:  If the oracle takes $p\in P$ as input, then it outputs ${\cal N}_{\tau,\delta, P}(p)$. 
\end{definition}

In the definition above, we formulate the $\anng$ by adding a noise probability $\delta$ in the construction of the graph. For a vertex on the graph, it has an out-edge to one of its near neighbors with probability $\delta<1$. This condition brings uncertainty to the existence of edges on graph, which is closer to the $\anng$ build by existing algorithms~\citep{dml11,zhg+13,fxw+17,bbm18,ztx+19,ztl20,txz+21}.

Next, we list the parameters required for the oracle defined on $\anng$ as below:

\begin{definition}[Oracle parameters]\label{def:oracle_param}
Let $\tau>1$, $\delta\in (0,1)$ denote two fix parameters. Let $P \subset \mathbb{S}^{d-1}$ denote a $n$-point dataset  where every $p\in P$ is i.i.d. sampled from ${\cal U}(\mathbb{S}^{d-1})$. Let ${\cal O}_{\tau,\delta,P}$ denote the oracle defined in Definition~\ref{def:graph_oracle}. ${\cal O}_{\tau,\delta,P}$ is associated with a $\anng$ $G_{\tau,\delta,P}$~(see  Definition~\ref{def:graph_oracle}). Let ${\cal N}_{\tau,\delta,P}(p)$ denote the neighbor set (see  Definition~\ref{def:graph_oracle}) for vertex $p\in P$ on $G_{\tau,\delta,P}$~(see  Definition~\ref{def:graph_oracle}). We define $N_{\tau,\delta,P}(p)$ as the size of the neighbor set ${\cal N}_{\tau,\delta,P}(p)$ for $p\in P$. We define $E_{\tau,\delta,P}$ as the number of edges on $G_{\tau,\delta,P}$.
\end{definition}

Next, we define the greedy step using oracle ${\cal O}_{\tau,\delta,P}$ defined in Definition~\ref{def:graph_oracle}.

\begin{definition}[Greedy step]\label{def:greedy_step}
With parameters defined in Definition~\ref{def:oracle_param},
for a query vector $q\in \mathbb{S}^{d-1}$, we define the greedy step for $q$ at $p\in P$ using oracle ${\cal O}_{\tau,\delta,P}$ as:
(1) Call oracle ${\cal O}_{\tau,\delta,P}$ with input $p$. Obtain ${\cal N}_{\tau,\delta,P}(p)$,  (2) If $\max_{p'\in {\cal N}_{\tau,\delta,P}(p)} \langle p',q \rangle>\langle p,q \rangle$, output $\arg\max_{p'\in {\cal N}_{\tau,\delta,P}(p)} \langle p',q \rangle$, otherwise, output $\mathsf{fail}$.
\end{definition}

\subsection{Algorithm}
In this section, we introduce the algorithms for $\nns$ using  oracle ${\cal O}_{\tau,\delta,P}$ defined in Definition~\ref{def:graph_oracle}. As shown in Algorithm~\ref{alg:nngraph_query}, we present a greedy search with three steps:
\begin{enumerate}
    \item starting at a vertex $p$, call oracle   ${\cal O}_{\tau,\delta,P}$ to obtain  ${\cal N}_{\tau,\delta,P}(p)$
    \item compute the sine distances between $q$ and every vector in ${\cal N}_{\tau,\delta,P}(p)$
    \item if we find a vector $p'$ in  ${\cal N}_{\tau,\delta,P}(p)$ that is closer to $q$ than $p$, set $p'$ to be the new $p$ and iterate again.
\end{enumerate}
Most of the current graph-based $\nns$ algorithms use this query algorithm in practice~\citep{fxw+17,my18,ztx+19,ztl20,txz+21}.

\ifdefined\isarxiv
\begin{algorithm}[H]
\caption{Query}
\label{alg:nngraph_query}
\begin{algorithmic}[1]
\Procedure{Query}{$q\in \R^d$, $\omega>1$, $r\in (1,2^{\omega})$, ${\cal O}_{\tau,\delta, P}$}\Comment{${\cal O}_{\tau,\delta, P}$ is defined in Definition~\ref{def:graph_oracle}}
\State $p\sim {\cal U}(P)$ \Comment{Uniformly random sample from set $P$}
\While{$\sin \theta_{p,q} > r\cdot 2^{-\omega}$} 
\State ${\cal N}_{\tau,\delta, P}(p)\leftarrow {\cal O}_{\tau,\delta, P}.\textsc{query}(p)$
\Comment{Call the oracle defined in Definition~\ref{def:graph_oracle}}
\State $p_{\textsf{tmp}}\leftarrow p$
\For{$p' \in {\cal N}_{\tau,\delta, P}(p)$} 
\If{$\langle p',q \rangle>\langle p_{\textsf{tmp}},q \rangle$}
\State $p_{\textsf{tmp}}\leftarrow p'$
\EndIf
\EndFor
\If{$p_{\textsf{tmp}}\neq p$}
\State $p\leftarrow p_{\textsf{tmp}}$\Comment{Make one step progress}
\Else
\State \Return $\textsf{Fail}$
\EndIf
\EndWhile
\State \Return $p$
\EndProcedure
\end{algorithmic}
\end{algorithm}
\else
\begin{algorithm}[!ht]
\caption{Query}
\label{alg:nngraph_query}
\begin{algorithmic}
\State {\bfseries Input:} $q\in \R^d$, $\omega>1$, $r\in (1,2^{\omega})$, ${\cal O}_{\tau,\delta, P}$\\
\Comment{${\cal O}_{\tau,\delta, P}$ is defined in Definition~\ref{def:graph_oracle}}
\State $p\sim {\cal U}(P)$ 
\Comment{Uniformly random sample from set $P$}
\While{$\sin \theta_{p,q} > r\cdot 2^{-\omega}$}
\State ${\cal N}_{\tau,\delta, P}(p)\leftarrow {\cal O}_{\tau,\delta, P}.\textsc{query}(p)$
\Comment{Call the oracle defined in Definition~\ref{def:graph_oracle}}
\State $p_{\textsf{tmp}}\leftarrow p$
\For{$p' \in {\cal N}_{\tau,\delta, P}(p)$} 
\If{$\langle p',q \rangle>\langle p_{\textsf{tmp}},q \rangle$}
\State $p_{\textsf{tmp}}\leftarrow p'$
\EndIf
\EndFor
\If{$p_{\textsf{tmp}}\neq p$}
\State $p\leftarrow p_{\textsf{tmp}}$
\Comment{Make one step progress}
\Else
\State {\bfseries Return:} $\textsf{Fail}$
\EndIf
\EndWhile 
\State \textbf{Return} $p$
\end{algorithmic}
\end{algorithm}
\fi

\subsection{Running Time of One Greedy Step}
In this section, we estimate the running time of one greedy step on the $\anng$. To start with, we estimate the number of neighbors for a vertex on the $\anng$.

\begin{lemma}[Estimation of number of neighbors]\label{lem:num_neighbor}
With the parameters defined in Definition~\ref{def:oracle_param}, we show that
with probability at least $1-O(\frac{1}{n\vol_c(\alpha_\tau)})$, the number of edges $N_{\tau,\delta,P}(p)$ is $O(n\vol_c(\alpha_\tau))$.
\end{lemma}
\begin{proof}
Following definition of the oracle ${\cal N}_{\tau,\delta,P}(p)$ (see Definition~\ref{def:graph_oracle}), we start with showing that 
\begin{align}\label{eq:edge_exist_prob}
    \Pr[p'\in {\cal U}(\mathbb{S}^{d-1}) \ \textit{and} \ p'\in {\cal N}_{\tau,\delta,P}(p) ]\in \big[\delta,1\big]\cdot\vol_c(\alpha_\tau).
\end{align}

Next, because every $p\in P$ is i.i.d. sampled from ${\cal U}(\mathbb{S}^{d-1})$, $N_{\tau,\delta,P}(p)$ follows the Binomial distribution with $n-1$ trials and each trails has success probability equal to $\Pr[p'\in {\cal U}(\mathbb{S}^{d-1}) \ \textit{and} \ p'\in {\cal N}_{\tau,\delta,P}(p) ]$. Therefore, the expectation of $N_{\tau,\delta,P}(p)$ could be bounded as: 
\begin{align}\label{eq:neighbors_num_expect}
    \E[N_{\tau,\delta,P}(p)]\in \big[\delta,1\big]\cdot(n-1)\vol_c(\alpha_\tau).
\end{align}

Next, following Corollary~\ref{coro:chebyshev_extend}, we show that,
\begin{align}\label{eq:contentration_num_neigbors}
     \Pr\big[|N_{\tau,\delta,P}(p)-\E[N_{\tau,\delta,P}(p)]|> 0.5  (n-1)\vol_c(\alpha_\tau)  \big] 
       <  \frac{4}{(n-1)\vol_c(\alpha_\tau)}.
\end{align}

Combining Eq.~\eqref{eq:contentration_num_neigbors} and Eq.~\eqref{eq:neighbors_num_expect}, we prove that with probability at least $1- \frac{4}{(n-1)\vol_c(\alpha_\tau)}$, the expected number of neighbors follows $$N_{\tau,\delta,P}(p) \in[1-0.5\delta,1.5]\cdot (n-1)\vol_c(\alpha_\tau).$$
For simplicity we could write $N_{\tau,\delta,P}(p)$ as $O(n\vol_c(\alpha_\tau))$.
\end{proof}

Next, we provide the running time of one greedy step as below:

\begin{lemma}\label{lem:1-steps}
With parameters defined in Definition~\ref{def:oracle_param}, for a query vector $q\in \mathbb{S}^{d-1}$, we show that with probability at least $1 - O(\frac{1}{ n\cdot\vol_c(\alpha_\tau)})$, the time complexity of a greedy step (Definition~\ref{def:greedy_step}) for $q$ at $p\in P$ using oracle ${\cal O}_{\tau,\delta,P}$ is $ O(nd\cdot\vol_c(\alpha_\tau))$.
\end{lemma}
\begin{proof}
The greedy step defined in Definition~\ref{def:greedy_step} requires comparison of $q$ with all vector in the $N_{\tau,\delta,P}(p)$. Following Lemma~\ref{lem:num_neighbor}, we know that we probability at least $1 - O(\frac{1}{ n\cdot\vol_c(\alpha_\tau)})$, it requires $O(n\vol_c(\alpha_\tau))$ number of distance computations. As each distance computation takes $O(d)$, we prove the final complexity.
\end{proof}

\subsection{Space of the Algorithm}

In this section, we provide the space complexity of Algorithm~\ref{alg:nngraph_query} in real RAM. To perform the query as shown in Algorithm~\ref{alg:nngraph_query}, we need to store the whole $\anng$ and the whole dataset. Storing the dataset $P$ takes $O(nd)$ space. Meanwhile, the space complexity for a graph is determined by the number of edges in this graph. Thus, we provide an estimation of the number of edges as below.
\begin{lemma}[Estimate of number of edges]\label{lem:space}
With the parameters defined in Definition~\ref{def:oracle_param}, we show that with probability $1 - O(\frac{1}{n^2\cdot\vol_c(\alpha_\tau)})$, the number of edges $ E_{\tau,\delta,P}$ is $O(n^2\cdot \vol_c(\alpha_\tau))$.
\end{lemma}
\begin{proof}
From Eq.~\eqref{eq:edge_exist_prob}, we know that an edge from vertex $p$ to vertex $p'$ exist with probability lower than $\vol_c(\alpha_\tau)$ but higher than $\delta\cdot \vol_c(\alpha_\tau)$. Therefore,  the expected number for $E_{\tau,\delta,P}$ follows the Binomial distribution with $n(n-1)$ trials and each trails has success probability equal to $\Pr[p'\in {\cal U}(\mathbb{S}^{d-1}) \ \textit{and} \ p'\in {\cal N}_{\tau,\delta,P}(p) ]$. Following this distribution, we write the expected number of edges as
\begin{align}\label{eq:edges_num_expect}
   \E[ E_{\tau,\delta,P}]\in \big[\delta ,1\big]\cdot 
   n(n-1)\vol_c(\alpha_\tau).
\end{align}

Next, following Corollary~\ref{coro:chebyshev_extend}, we show that,
\begin{align}\label{eq:contentration_num_edges}
      \Pr\big[| E_{\tau,\delta,P}-\E[ E_{\tau,\delta,P}]|> 0.5  n(n-1) \vol_c(\alpha_\tau)\big] 
     <  \frac{4}{n(n-1)\vol_c(\alpha_\tau)}.
\end{align}

Combing Eq.~\eqref{eq:contentration_num_edges} and Eq.~\eqref{eq:edges_num_expect}, with probability $1 - O(\frac{1}{n^2\cdot\vol_c(\alpha_\tau)})$, we show that the number of edges
$ E_{\tau,\delta,P}$ is $O(n^2\vol_c(\alpha_\tau))$. 
\end{proof}

Therefore, with probability $1 - O(\frac{1}{n^2\cdot\vol_c(\alpha_\tau)})$, the space complexity for $\anng$ $G_{\tau,\delta,P}$~(see  Definition~\ref{def:graph_oracle}) is $O(nd+n^2\cdot \vol_c(\alpha_\tau))$.

\section{Guarantees}\label{sec:thoery_guanratee}
In this section, we provide the theoretical guarantees for solving $\nns$ problem via Algorithm~\ref{alg:nngraph_query}. We start with presenting the supporting lemmas. Then, we introduce our main theorem with proof.

\subsection{Lower Bound of Wedge Volume}
In this section, we provide the lower bound for the volume for wedge $\vol_w(\alpha_\tau, \alpha_s, \arcsin((s+ \varepsilon)2^{-\omega}))$.   

\begin{lemma}[Lower bound of wedge volume]\label{lem:varepsilon}
Let function $\alpha$ be defined as Definition~\ref{def:alpha_function}. Let $s>1$ denote a parameter. Let $\varepsilon>0$. If $\tau \geq \sqrt{2}(s+\epsilon)$, then
$\vol_w(\alpha_\tau, \alpha_s, \arcsin((s+ \varepsilon)2^{-\omega})) \geq \frac{s^{d}}{n\sqrt{d}}$.
\end{lemma}
\begin{proof}
We start with showing that
\begin{align}\label{eq:alpha_m_s_eps}
    \alpha_\tau
    = & ~ \sqrt{1-\tau^2\cdot 2^{-2\omega}}\notag\\
    \leq & ~ \sqrt{1-2(s+\epsilon)^2\cdot 2^{-2\omega}}\notag\\
    <  & ~  \sqrt{1-s^2\cdot 2^{-2\omega}-(s+\epsilon)^2\cdot 2^{-2\omega}}\notag\\
    <  & ~ \sqrt{1-s^22^{-2\omega}}\cdot \sqrt{1-(s+\epsilon)^2 2^{-2\omega}}\notag\\
    = & ~ \alpha_s\alpha_{s+\epsilon},
\end{align}
where the second step follows from $\tau \geq \sqrt{2}(s+\epsilon)$, the third step follows from $\epsilon>0$, the forth step follows from $\sqrt{1-a-b}<\sqrt{1-a}\cdot \sqrt{1-b}$ for all $a,b\in [0,1]$, the fifth step follows the definition of $\alpha$ function in Definition~\ref{def:alpha_function}.

Next, we present the lower bound of the wedge $\vol_w(\alpha_\tau, \alpha_s, \arcsin((s+ \varepsilon)2^{-\omega}))$ as:
\begin{align*}
     \vol_w(\alpha_\tau, \alpha_s, \arcsin((s+ \varepsilon)2^{-\omega})) 
    > & ~ \frac{1}{2} \vol_c(\alpha_s)\\
    > & ~ \Theta( d^{-1/2} ) (\sqrt{1 - \alpha_s^2})^{d}\\
    = & ~ \Theta( d^{-1/2} ) s^{d}2^{-d\omega}\\
    = & ~ \Theta( d^{-1/2} ) \cdot \frac{s^{d}}{n}\\
    \geq & ~ \frac{s^{d}}{n\sqrt{d}},
\end{align*}
where the first step follows from Lemma.~\ref{lem:prep_vol_wedge}, the second step follows from Lemma~\ref{lem:vol_cap}, the third step follows from the definition of $\alpha$ function, the forth step follows from $2^{-d\omega}=\frac{1}{n}$, the fifth step follows from the definition of $\Theta(d^{-1/2})$.
\end{proof}

\subsection{Probability of Making Progress}
In this section, we present the lower bound for the probability of making a step progress towards the nearest neighbor. 
\begin{lemma}[]\label{lem:progress}
With parameters defined in Definition~\ref{def:oracle_param},
given a query vector $q\in \mathbb{S}^{d-1}$, we show that for any $p_1\in P$ such that $\sin \theta_{p_1,q}= (s+\varepsilon)2^{-\omega}$, if $\tau\geq \sqrt{2}(s+\epsilon)$, then
\begin{align*}
    \Pr\big[\exists p_{2}\in P, \ p_{2}\in{\cal N}_{\tau,\delta,P}(p_1) \ \textit{and} \ \sin{\theta_{p_2,q}}\leq s\cdot2^{-\omega}\big] \geq 1 - O(e^{-s^{d}\delta/\sqrt{d}}).
\end{align*}
\end{lemma}

\begin{proof}
We start with lower bounding the probability that a data point $p_2\in {\cal U}({\cal S}^{d-1})$ satisfies $\langle p_2, p_1 \rangle\geq \alpha_\tau$ and $\sin{\theta_{p_2,q}}\leq s\cdot2^{-\omega}$.

\begin{align}\label{eq:one_point_progress}
     \Pr\big[p_2\in {\cal U}({\cal S}^{d-1}), \ \langle p_{2},p_1 \rangle\geq \alpha_\tau , \sin{\theta_{p_2,q}}\leq s\cdot2^{-\omega}\big] 
     = & ~ \vol_w(\alpha_\tau, \alpha_s, \arcsin((s+ \varepsilon)2^{-\omega}))\notag\\
     > & ~ \frac{s^{d}}{n\sqrt{d}},
\end{align}
where the first step follows from the definition of the volume of wedge in Definition~\ref{def:volume_wedge}, the second step follows from Lemma~\ref{lem:varepsilon}.

Next, we lower bound the probability that a data point $p_2\in {\cal U}({\cal S}^{d-1})$ satisfies $\langle p_2, p_1 \rangle\geq \alpha_\tau$ and $\sin{\theta_{p_2,q}}\leq s\cdot2^{-\omega}$ stays in the ${\cal N}_{\tau,\delta,P}(p_1)$.
\begin{align}\label{eq:one_point_noise_progress}
    & ~\Pr\big[  p_2\in {\cal U}({\cal S}^{d-1}), \ p_{2}\in{\cal N}_{\tau,\delta,P}(p_1)~ \textit{and} 
    ~\ \sin{\theta_{p_2,q}}\leq s\cdot2^{-\omega}\big] \notag \\
     > & ~ \delta \Pr\big[p_2\in {\cal U}({\cal S}^{d-1}), \ \langle p_{2},p_1 \rangle\geq \alpha_\tau~\textit{and} ~   \sin{\theta_{p_2,q}}\leq s\cdot2^{-\omega}\big]  \notag\\
      > & ~ \frac{s^{d}\delta}{n\sqrt{d}},
\end{align}
where the first step follows from the definition of ${\cal N}_{\tau,\delta,P}(p_1)$ in Definition~\ref{def:oracle_param}, the second step follows from Eq.~\eqref{eq:one_point_progress}.

Next, we could lower bound the probability that there exists a $p_2\in P$ that satisfies $p_{2}\in{\cal N}_{\tau,\delta,P}(p_1)$ and $\sin{\theta_{p_2,q}}\leq s\cdot2^{-\omega}$ by union bounding the probability in Eq.~\eqref{eq:one_point_noise_progress}.
\begin{align*}
   \Pr\big[\exists p_{2}\in P, \ p_{2}\in{\cal N}_{\tau,\delta,P}(p_1) \ \textit{and} \ \sin{\theta_{p_2,q}}\leq s\cdot2^{-\omega}\big]
     & ~ > 1 - (1 - \frac{s^{d}\delta}{n\sqrt{d}})^{n-1}  \\
     & ~ = 1 - O(e^{-s^{d}\delta/\sqrt{d}}),
\end{align*}
where the first step follow from Eq.~\eqref{eq:one_point_noise_progress}, the second step is a reorganization. 
\end{proof}

\subsection{Main Result}
Here is our main theorem to prove the number of steps required to reach the nearest neighbor.

\begin{theorem}[A formal version of Theorem~\ref{thm:main:informal}]\label{thm:main:formal}

Let $\omega>1$ and $r\in (1,2^{\omega})$ denote two parameters.  Let $r_0>r$ and  $\epsilon\in (0, r_0-r)$. Let $\tau\geq \sqrt{2}r_0$ and $\delta\in (0,1)$. Let $P \subset \mathbb{S}^{d-1}$ denote a $n$-point, $\omega$-dense dataset (see Definition~\ref{def:dense_dataset})  where every $p\in P$ is i.i.d. sampled from ${\cal U}(\mathbb{S}^{d-1})$. Let ${\cal O}_{\tau,\delta,P}$ denote the oracle defined in Definition~\ref{def:graph_oracle}. ${\cal O}_{\tau,\delta,P}$ is associated with a $\anng$ $G_{\tau,\delta,P}$~(see  Definition~\ref{def:graph_oracle}). Let ${\cal N}_{\tau,\delta,P}(p)$ denote the neighbor set (see  Definition~\ref{def:graph_oracle}) for vertex $p\in P$ on $G_{\tau,\delta,P}$~(see  Definition~\ref{def:graph_oracle}). We define $N_{\tau,\delta,P}(p)$ as the size of the neighbor set ${\cal N}_{\tau,\delta,P}(p)$ for $p\in P$. We define $E_{\tau,\delta,P}$ as the number of edges on $G_{\tau,\delta,P}$.

Given a query $q\in\mathbb{S}^{d-1}$ with the premise that $\min_{p\in P}\sin{\theta_{p,q}}\leq r\cdot 2^{-\omega}$, there is an algorithm (Algorithm~\ref{alg:nngraph_query}) using ${\cal O}_{\tau,\delta,P}$ that takes $O(2^{\omega}\epsilon^{-1}d^{1/2}\tau^d)$ query time, $O(n^2d\delta)$ preprocessing time and $O(n(\tau^d+d))$ space, starting from an $p_0\in P$ with $\sin \theta_{p_0,q}\leq r_0 2^{-\omega}$, that solves $r$-$\nns$ problem (Definition~\ref{def:nn:formal}) using $T$ steps with probability $1 - T \cdot e^{-r^{d}\delta/\sqrt{d}}$.
\end{theorem}

\begin{proof}
We start with showing how to make $\epsilon\cdot 2^{-\omega}$ progress from $p_0$. Let $s=r_0-\epsilon$ so that we have $\tau\geq\sqrt{2} (s+\epsilon)$. Then, we have
\begin{align}\label{eq:induct_first}
      \Pr\big[\exists p_{1}\in P, \ p_{1}\in{\cal N}_{\tau,\delta,P}(p_0) \  \textit{and}  ~ \sin{\theta_{p_1,q}} \leq (r_0-\epsilon)\cdot2^{-\omega}\big] 
     & \geq ~ 1 - O(e^{-(r_0-\epsilon)^{d}\delta/\sqrt{d}})\notag\\
     & > ~ 1-O(e^{-r^{d}\delta/\sqrt{d}}),
\end{align}
where the first step follows from Lemma~\ref{lem:progress}, the second step follows from $r_0-\epsilon>r$.

Then, applying induction over Eq.~\eqref{eq:induct_first}, we show that for $i\in \mathbb{N}$,
\begin{align}\label{eq:induct}
   \Pr\big[\exists p_{i+1}\in P, \  p_{i+1} \in{\cal N}_{\tau,\delta,P}(p_i) \ \textit{and} ~ \sin{\theta_{p_{i+1},q}}\leq (r_0-i\cdot \epsilon)\cdot2^{-\omega}\big] 
    >  1 -O(e^{-r^{d}\delta/\sqrt{d}}).
\end{align}

Therefore, to reach a $p_T\in P$ such that $\sin{\theta_{p_T,q}}\leq r\cdot2^{-\omega}$, we should take $T=(r_0-r)2^{\omega}/\epsilon=O(2^{\omega}/\epsilon)$ steps. Moreover, we could union bound the total failure probability as $T \cdot e^{-r^{d}\delta/\sqrt{d}} $. 

Next, we analyze the complexity for query time, preprocessing time and space.

{\bf Query time:}
From Lemma~\ref{lem:1-steps}, we know that for a query vector $q\in \mathbb{S}^{d-1}$, with probability $1 - O(\frac{1}{ n\delta\cdot\vol_c(\alpha_\tau)})$, the time complexity of a greedy step (Definition~\ref{def:greedy_step}) for $q$ at $p$ using oracle ${\cal O}_{\tau,\delta,P}$ (see Definition~\ref{def:graph_oracle}) is $ (n-1)d\cdot\vol_c(\alpha_\tau)$. Following this lemma, we write the total query time is
\begin{align*}
     T(n-1)d\cdot\vol_c(\alpha_\tau) 
      = & ~  O(2^{\omega}\epsilon^{-1}\cdot (n-1)d \cdot d^{-1/2} \tau^d2^{-d\omega} )\\
     = & ~ O(2^{\omega}\epsilon^{-1}d^{1/2}\tau^d),
\end{align*}
where he first step follows Lemma~\ref{lem:prep_vol_wedge}, the second step follows from $2^{-d\omega}=\frac{1}{n}$ in Fact~\ref{fact:omega_to_d}.

Next, we union bound the failure probability of achieving $O(2^{\omega}\epsilon^{-1}d^{1/2}\tau^d\delta)$ query time complexity as 
\begin{align*}
    O(\frac{T}{ n\cdot\vol_c(\alpha_\tau)})
    = & ~ O(\frac{2^{\omega}\epsilon^{-1}}{ n\cdot d^{-1/2} \tau^d2^{-d\omega}})\\
    = & ~ O(2^{\omega}\epsilon^{-1}d^{1/2}\tau^{-d}),
\end{align*}
where the first step follows Lemma~\ref{lem:prep_vol_wedge},the second step follows from $2^{-d\omega}=\frac{1}{n}$ in Fact~\ref{fact:omega_to_d}.

Therefore, with probability $1-O(2^{\omega}\epsilon^{-1}d^{1/2}\tau^{-d})$, the query complexity is $O(2^{\omega}\epsilon^{-1}d^{1/2}\tau^d)$.

{\bf Preprocessing time:}
Build an oracle $O_{\tau,\delta,P}$ (see Definition~\ref{def:graph_oracle}) requires compute pairwise distances for vectors $p\in P$. As $\anng$ only samples a subset of $P$ with size $O(n\delta)$ for any $p\in P$ to compute, the preprocessing time is $O(n^2d\delta)$. 

{\bf  Space:}
We need to store all the vectors in $P$, which takes $O(nd)$ space. Moreover, following Lemma~\ref{lem:space}, the number of edges for $G_{\tau,\delta,P}$ in oracle $O_{\tau,\delta,P}$ (see Definition~\ref{def:graph_oracle}) is:
\begin{align*}
    E_{\tau,\delta,P} 
     = & ~O( n^2 \cdot \vol_c(\alpha_\tau))\\
  = & ~O(  n^2\cdot d^{-1/2} \tau^d2^{-d\omega})\\
   = & ~O(n d^{-1/2} \tau^d),
\end{align*}
where the first step follows Lemma~\ref{lem:prep_vol_wedge},the second step follows from $2^{-d\omega}=\frac{1}{n}$ in Fact~\ref{fact:omega_to_d}.

Next, we write the failure probability as $ O(\frac{1}{n^2\vol_c(\alpha_\tau)})=O(\frac{1}{nd^{-1/2} \tau^d})$. Thus, with probability $1-O(\frac{1}{nd^{-1/2} \tau^d})$, the space complexity is $O(nd^{-1/2} \tau^d+nd)=O(n(\tau^d+d))$. 
\end{proof}

\subsection{Discussion}
In Theorem~\ref{thm:main:informal}, the structure of the graph is determined by $\tau$ and $\delta$. For $p_1\in P$, a $p_2\in P$ that has $\langle p_1,p_2 \rangle \geq \alpha_{\tau}$ only has probability as least $\delta$ to be the out-neighbors of $p_1$ on the $G_{\tau,\delta,P}$. When we reach $p$ on the $G_{\tau,\delta,P}$, we search within ${\cal N}_{\tau,\delta,P}(p)$ with a hope to retrieve a vector $p'$ that satisfies $\sin \theta_{p',q}\leq \sin \theta_{p,q} -\epsilon\cdot 2^{-\omega}$. From the Theorem~\ref{thm:main:formal}, we know that as $\delta$ decreases, while the greedy search time complexity remains the same, we would have higher the failure probability.

Here the $\delta$ could be data-dependent. We show that for any vector $p_j$ and $p_j$, we could set head probability $\delta_{i,j}$ for coin $c_{i,j}$ in Oracle ${\cal O}_{\tau,\delta, P}$ (see ~\ref{def:graph_oracle}) as $1-\pi^{-1}\theta_{p_i,p_j}$. This corresponds to typical scenarios where we index the $\anng$ using SimHash~\cite{c02,zhg+13} or Product Quantization~\cite{jds10,bbm18}. In this case, Algorithm~\ref{alg:nngraph_query} would have the same query complexity and space with greedy search on the exact near neighbor graph, but the failure probability would be $O(T \cdot e^{-r^{d}(1-\pi^{-1}\arccos{\alpha_{\tau}})/\sqrt{d}})$. 

Although our paper focuses on the theoretical analysis, the implementation of our graph-based algorithm leads to potential carbon dioxide release. On the other hand, we hope our method could provide guidance to practitioners to reduce their repetition in experiments.

\section{More Results}
In this section, we introduce the extended results that describes trade-offs for the second and third cases in the introduction. We start with introducing the $\anng$ algorithm. Next, we analyze the time and space complexity of this algorithm. Finally, we introduce the guarantees for this algorithm on $\nns$. 

\subsection{Algorithm}\label{sec:sup_alg}
\subsubsection{Definitions}
We start with defining the formal version of the graph oracle. 
\begin{definition}[Oracle]\label{def:graph_oracle_twoside:formal}
Let $\tau>1$, $\omega>1$, $\delta_1\in (0,1)$ and  $\delta_2\in (0,1)$ denote four fixed parameters. Let $\delta_1>\delta_2$. Let function $\alpha$ with parameter $\omega$ be defined as in Definition~\ref{def:alpha_function}. Let $P \subset \mathbb{S}^{d-1}$ denote a $n$-point, $\omega$-dense (Definition~\ref{def:dense_dataset}) dataset  where every $p\in P$ is i.i.d. sampled from ${\cal U}(\mathbb{S}^{d-1})$. Let $\{c^1_{i,j}|\ i,j\in[n]\ \textit{and} \ i\neq j  \}$ denote a set of $n(n-1)$ biased coins. When we toss any coin $c^1_{i,j}$ in this set where $i,j\in[n]$ and $i\neq j$, it shows head with probability $\delta_1$. Let $\{c^2_{i,j}|\ i,j\in[n]\ \textit{and} \ i\neq j  \}$ denote a set of $n(n-1)$ biased coins. When we toss any coin $c^2_{i,j}$ in this set where $i,j\in[n]$ and $i\neq j$, it shows head with probability $\delta_2$. 

We define the $\anng$ $G_{\tau,\delta_1,\delta_2, P}$ as a directed graph over $P$. $G_{\tau,\delta_1,\delta_2, P}$ has $n$ vertices, where each vertex represents a datapoint $p\in P$. There are no self-loops on $G_{\tau,\delta_1,\delta_2, P}$. For $i,j\in[n]$ and $i\neq j$, an edge from vertex $p_i$ to $p_j$ exists on $G_{\tau,\delta_1,\delta_2, P}$ if and only if $p_i$ and $p_j$ satisfies any one of the following conditions:
\begin{itemize}
    \item  $\langle p_i,p_j \rangle \geq \alpha_{\tau}$ and we flip the coin $c^1_{i,j}$ for one time with the results being head.
    \item $\langle p_i,p_j \rangle \leq \alpha_{\tau}$ and we flip the coin $c^2_{i,j}$ for one time with the results being head.
\end{itemize}

For any $p\in P$, we define a neighbor set ${\cal N}_{\tau,\delta_1,\delta_2, P}(p)\subset P$ as: for every $p'\in {\cal N}_{\tau,\delta_1,\delta_2, P}(p)$, there exists an edge from $p$ to $p'$ on $G_{\tau,\delta_1,\delta_2, P}$.

We define an oracle ${\cal O}_{\tau,\delta_1,\delta_2, P}$ as:  If the oracle takes $p\in P$ as input, then it outputs ${\cal N}_{\tau,\delta_1,\delta_2, P}(p)$. 
\end{definition}

The oracle may relates to a set of $\anng$s built via K-means~\cite{mrs20} or sketching~\cite{sy21,swy+21}. Next, we the parameters required for the oracle (see Definition~\ref{def:graph_oracle_twoside:formal}) defined on $\anng$ as below:

\begin{definition}[Oracle parameters]\label{def:oracle_param_twoside}
Let $\tau>1$, $\delta_1\in (0,1)$, $\delta_2$ denote three fix parameters. Let $\delta_1>\delta_2$. Let $P \subset \mathbb{S}^{d-1}$ denote a $n$-point dataset  where every $p\in P$ is i.i.d. sampled from ${\cal U}(\mathbb{S}^{d-1})$. Let ${\cal O}_{\tau,\delta_1,\delta_2,P}$ denote the oracle defined in Definition~\ref{def:graph_oracle_twoside:formal}. ${\cal O}_{\tau,\delta_1,\delta_2,P}$ is associated with a $\anng$ $G_{\tau,\delta_1,\delta_2,P}$~(see  Definition~\ref{def:graph_oracle_twoside:formal}). Let ${\cal N}_{\tau,\delta_1,\delta_2,P}(p)$ denote the neighbor set (see  Definition~\ref{def:graph_oracle_twoside:formal}) for vertex $p\in P$ on $G_{\tau,\delta_1,\delta_2,P}$~(see  Definition~\ref{def:graph_oracle_twoside:formal}). We define $N_{\tau,\delta_1,\delta_2,P}(p)$ as the size of the neighbor set ${\cal N}_{\tau,\delta_1,\delta_2,P}(p)$ for $p\in P$. We define $E_{\tau,\delta_1,\delta_2,P}$ as the number of edges on $G_{\tau,\delta_1,\delta_2,P}$.
\end{definition}

Next, we define the greedy step using oracle ${\cal O}_{\tau,\delta_1,\delta_2,P}$ defined in Definition~\ref{def:graph_oracle_twoside:formal}.

\begin{definition}[Greedy step]\label{def:greedy_step_twoside}
With parameters defined in Definition~\ref{def:oracle_param_twoside},
for a query vector $q\in \mathbb{S}^{d-1}$, we define the greedy step for $q$ at $p\in P$ using oracle ${\cal O}_{\tau,\delta_1,\delta_2,P}$ as:
\begin{enumerate}
    \item Call oracle ${\cal O}_{\tau,\delta_1,\delta_2,P}$ with input $p$. Obtain ${\cal N}_{\tau,\delta_1,\delta_2,P}(p)$
    \item If $\max_{p'\in {\cal N}_{\tau,\delta_1,\delta_2,P}(p)} \langle p',q \rangle>\langle p,q \rangle$, output $\arg\max_{p'\in {\cal N}_{\tau,\delta_1,\delta_2,P}(p)} \langle p',q \rangle$, otherwise, output $\mathsf{fail}$.
\end{enumerate}
\end{definition}

\subsubsection{Greedy Search Algorithm}
In this section, we introduce the algorithms for $\nns$ using  oracle ${\cal O}_{\tau,\delta_1,\delta_2,P}$ defined in Definition~\ref{def:graph_oracle_twoside:formal}. As shown in Algorithm~\ref{alg:nngraph_query_twoside}, we present a greedy search with three steps: (1) starting at a vertex $p$, call oracle   ${\cal O}_{\tau,\delta_1,\delta_2,P}$ to obtain  ${\cal N}_{\tau,\delta_1,\delta_2,P}(p)$, (2) compute the sine distances between $q$ and every vector in ${\cal N}_{\tau,\delta_1,\delta_2,P}(p)$, (3) if we find a vector $p'$ in  ${\cal N}_{\tau,\delta_1,\delta_2,P}(p)$ that is closer to $q$ than $p$, set $p'$ to be the new $p$ and iterate again. Most of the current graph-based $\nns$ algorithms use this query algorithm in practice~\citep{fxw+17,my18,ztx+19,ztl20,txz+21}.

\ifdefined\isarxiv
\begin{algorithm}[H]
\caption{Query}
\label{alg:nngraph_query_twoside}
\begin{algorithmic}[1]
\Procedure{Query}{$q\in \R^d$, $\omega>1$, $r\in (1,2^{\omega})$, ${\cal O}_{\tau,\delta_1,\delta_2, P}$}\Comment{${\cal O}_{\tau,\delta_1,\delta_2, P}$ is defined in Definition~\ref{def:graph_oracle_twoside:formal}}
\State $p\sim {\cal U}(P)$ \Comment{Uniformly random sample from set $P$}
\While{$\sin \theta_{p,q} > r\cdot 2^{-\omega}$} 
\State ${\cal N}_{\tau,\delta_1,\delta_2, P}(p)\leftarrow {\cal O}_{\tau,\delta_1,\delta_2, P}.\textsc{query}(p)$
\Comment{Call the oracle defined in Definition~\ref{def:graph_oracle_twoside:formal}}
\State $p_{\textsf{tmp}}\leftarrow p$
\For{$p' \in {\cal N}_{\tau,\delta_1,\delta_2, P}(p)$} 
\If{$\langle p',q \rangle>\langle p_{\textsf{tmp}},q \rangle$}
\State $p_{\textsf{tmp}}\leftarrow p'$
\EndIf
\EndFor
\If{$p_{\textsf{tmp}}\neq p$}
\State $p\leftarrow p_{\textsf{tmp}}$\Comment{Make one step progress}
\Else
\State \Return $\textsf{Fail}$
\EndIf
\EndWhile
\State \Return $p$
\EndProcedure
\end{algorithmic}
\end{algorithm}
\else
\begin{algorithm}[H]
\caption{Query}
\label{alg:nngraph_query_twoside}
\begin{algorithmic}[1]
\State {\bfseries Input:} $q\in \R^d$, $\omega>1$, $r\in (1,2^{\omega})$, ${\cal O}_{\tau,\delta_1,\delta_2, P}$\\
\Comment{${\cal O}_{\tau,\delta_1,\delta_2, P}$ is defined in Definition~\ref{def:graph_oracle_twoside:formal}}
\State $p\sim {\cal U}(P)$ 
\Comment{Uniformly random sample from set $P$}
\While{$\sin \theta_{p,q} > r\cdot 2^{-\omega}$}
\State ${\cal N}_{\tau,\delta_1,\delta_2, P}(p)\leftarrow {\cal O}_{\tau,\delta_1,\delta_2, P}.\textsc{query}(p)$
\Comment{Call the oracle defined in Definition~\ref{def:graph_oracle_twoside:formal}}
\State $p_{\textsf{tmp}}\leftarrow p$
\For{$p' \in {\cal N}_{\tau,\delta_1,\delta_2, P}(p)$} 
\If{$\langle p',q \rangle>\langle p_{\textsf{tmp}},q \rangle$}
\State $p_{\textsf{tmp}}\leftarrow p'$
\EndIf
\EndFor
\If{$p_{\textsf{tmp}}\neq p$}
\State $p\leftarrow p_{\textsf{tmp}}$
\Comment{Make one step progress}
\Else
\State {\bfseries Return:} $\textsf{Fail}$
\EndIf
\EndWhile
\State \textbf{Return} $p$
\end{algorithmic}
\end{algorithm}
\fi

\subsection{Running Time of the Greedy Step}
In this section, we provide the running time of a greedy step on the $\anng$. 
\subsubsection{Estimation of number of neighbors}
We start with estimating the number of neighbors for a vertex on the $\anng$.

\begin{lemma}[Estimation of number of neighbors]\label{lem:num_neighbor_twoside}
With the parameters defined in Definition~\ref{def:oracle_param_twoside}, we show that $\E[N_{\tau,\delta_1,\delta_2,P}(p)]=(n-1)\vol_c(\alpha_\tau)\delta$. Moreover, with probability at least $1-\frac{4}{(n-1)(\delta_2+(\delta_1-\delta_2)\vol_c)}$, 
\begin{align*}
 N_{\tau,\delta_1,\delta_2,P}(p) \in [1/2, 3/2] \cdot (n-1)(\delta_2+(\delta_1-\delta_2)\vol_c).
\end{align*}
\end{lemma}
\begin{proof}
We start with showing that 
\begin{align}\label{eq:edge_exist_prob_twoside}
    \Pr[p'\in {\cal U}(\mathbb{S}^{d-1}) \ \textit{and} \ p'\in {\cal N}_{\tau,\delta_1,\delta_2,P}(p) ]
    =&~\vol_c(\alpha_\tau)\delta_1+(1-\vol_c(\alpha_\tau))\delta_2\notag\\
    =&~\delta_2+(\delta_1-\delta_2)\vol_c.
\end{align}

Next, because every $p\in P$ is i.i.d. sampled from ${\cal U}(\mathbb{S}^{d-1})$, $N_{\tau,\delta_1,\delta_2,P}(p)$ follows the Binomial distribution with $n-1$ trials and each trails has success probability equal to $\Pr[p'\in {\cal U}(\mathbb{S}^{d-1}) \ \textit{and} \ p'\in {\cal N}_{\tau,\delta_1,\delta_2,P}(p) ]$. We write this Binomial distribution as $B(n-1,\delta_2+(\delta_1-\delta_2)\vol_c)$. Therefore, we write the expectation for $N_{\tau,\delta_1,\delta_2,P}(p)$ as:
\begin{align*}
    \E[N_{\tau,\delta_1,\delta_2,P}(p)]=(n-1)(\delta_2+(\delta_1-\delta_2)\vol_c).
\end{align*}
we write the variance for $N_{\tau,\delta_1,\delta_2,P}(p)$ as
\begin{align}~\label{eq:var_binomial_twoside}
    \Var[N_{\tau,\delta_1,\delta_2,P}(p)]
    =(n-1)(\delta_2+(\delta_1-\delta_2)\vol_c)(1-\delta_2-(\delta_1-\delta_2)\vol_c).
\end{align}

Next, we show that
\begin{align*}
     & ~ \Pr\big[|N_{\tau,\delta_1,\delta_2,P}(p)-\E[N_{\tau,\delta_1,\delta_2,P}(p)]| > \frac{1}{2}(n-1)(\delta_2+(\delta_1-\delta_2)\vol_c)\big] \\
       \leq  & ~ \frac{4 \Var[N_{\tau,\delta_1,\delta_2,P}(p)]}{(n-1)^2(n-1)^2(\delta_2+(\delta_1-\delta_2)\vol_c)^2}\\
      = & ~ \frac{4 (1-\delta_2-(\delta_1-\delta_2)\vol_c)}{(n-1)(\delta_2+(\delta_1-\delta_2)\vol_c)}\\
      < & ~ \frac{4}{(n-1)(\delta_2+(\delta_1-\delta_2)\vol_c)},
\end{align*}
where the first step follows the Chebyshev's inequality (see Lemma~\ref{lem:chebyshev}), the second step follows Eq.~\eqref{eq:var_binomial_twoside}, the third step follows $\delta_2+(\delta_1-\delta_2)\vol_c\in (0,1)$.

Therefore, we prove that with probability at least $1-\frac{4}{(n-1)(\delta_2+(\delta_1-\delta_2)\vol_c)}$, the number of neighbors $N_{\tau,\delta,P}(p) \in [1/2, 3/2] \cdot (n-1)(\delta_2+(\delta_1-\delta_2)\vol_c)$.
\end{proof}

\subsubsection{Running Time of One Greedy Step}

Next, we give the running time of one greedy step on the $\anng$.
\begin{lemma}\label{lem:1-steps_twoside}
With parameters defined in Definition~\ref{def:oracle_param_twoside}, for a query vector $q\in \mathbb{S}^{d-1}$, we show that with probability at least $1 - O(\frac{1}{ n(\delta_2+(\delta_1-\delta_2)\vol_c)})$, the time complexity of a greedy step (Definition~\ref{def:greedy_step_twoside}) for $q$ at $p\in P$ using oracle ${\cal O}_{\tau,\delta,P}$ is $ O(nd(\delta_2+(\delta_1-\delta_2)\vol_c))$.
\end{lemma}
\begin{proof}
The greedy step defined in Definition~\ref{def:greedy_step_twoside} requires comparison of $q$ with all vector in the $N_{\tau,\delta_1,\delta_2,P}(p)$. Following Lemma~\ref{lem:num_neighbor_twoside}, we know that we probability at least $1 - O(\frac{1}{ n(\delta_2+(\delta_1-\delta_2)\vol_c)})$, it requires $O(n(\delta_2+(\delta_1-\delta_2)\vol_c))$ number of distance computations. As each distance computation takes $O(d)$, we prove the final complexity.
\end{proof}

\subsection{Space of the Algorithm}

In this section, we provide the space complexity of Algorithm~\ref{alg:nngraph_query_twoside} in real RAM. To perform the query as shown in Algorithm~\ref{alg:nngraph_query_twoside}, we need to store the whole $\anng$ and the whole dataset. Storing the dataset $P$ takes $O(nd)$ space. Meanwhile, the space complexity for a graph is determined by the number of edges in this graph. Thus, we provide an estimation of the number of edges as below.
\begin{lemma}[Estimate of number of edges]\label{lem:space_twoside}
With the parameters defined in Definition~\ref{def:oracle_param_twoside}, we show that the expected value of the number of edges $E_{\tau,\delta_1,\delta_2,P}$ is $n(n-1)(\delta_2+(\delta_1-\delta_2)\vol_c)$. Moreover, with probability $1 - O(\frac{1}{n^2(\delta_2+(\delta_1-\delta_2)\vol_c)})$,
\begin{align*}
    E_{\tau,\delta_1,\delta_2,P} \in [1/4, 3/4] \cdot  n(n-1)(\delta_2+(\delta_1-\delta_2)\vol_c) .
\end{align*}
\end{lemma}
\begin{proof}
From Eq.~\eqref{eq:edge_exist_prob_twoside}, we know that an edge from vertex $p$ to vertex $p'$ exist with probability $(\delta_2+(\delta_1-\delta_2)\vol_c)$. Therefore,  the expected number for $E_{\tau,\delta_1,\delta_2,P}$ follows the Binomial distribution $B(n(n-1),(\delta_2+(\delta_1-\delta_2)\vol_c))$. Following this distribution, we write the expected number of edges as
\begin{align}\label{eq:edges_num_expect_twoside}
   \E[ E_{\tau,\delta_1,\delta_2,P}]=n(n-1)(\delta_2+(\delta_1-\delta_2)\vol_c).
\end{align}
We could also write the variance of the number edges $E_{\tau,\delta_1,\delta_2,P}$ as:
\begin{align}\label{eq:edges_num_var_twoside}
   \Var[ E_{\tau,\delta_1,\delta_2,P}]=n(n-1)(\delta_2+(\delta_1-\delta_2)\vol_c) (1-\delta_2-(\delta_1-\delta_2)\vol_c).
\end{align}

Next,  we show that

\begin{align*}
      \Pr\big[| E_{\tau,\delta_1,\delta_2,P}-\E[ E_{\tau,\delta_1,\delta_2,P}]|> \E[ E_{\tau,\delta_1,\delta_2,P}]/2\big] 
      & \leq ~ \frac{4\Var[ E_{\tau,\delta_1,\delta_2,P}]}{(\E[ E_{\tau,\delta_1,\delta_2,P}])^2}\\
      & = ~\frac{4(1-\delta_2-(\delta_1-\delta_2)\vol_c)}{n(n-1)(\delta_2+(\delta_1-\delta_2)\vol_c)}\\
      & < ~ \frac{4}{n(n-1)(\delta_2+(\delta_1-\delta_2)\vol_c)},
\end{align*}
where the first step follows the Chebyshev's inequality (see Lemma~\ref{lem:chebyshev}), the second step follows Eq.~\eqref{eq:edges_num_var_twoside}, the third step follows $(\delta_2+(\delta_1-\delta_2)\vol_c)\in (0,1)$.

Thus, with probability $1 - O(\frac{1}{n^2(\delta_2+(\delta_1-\delta_2)\vol_c)})$, we show that
$ E_{\tau,\delta_1,\delta_2,P} \in [1/4, 3/4] \cdot  n(n-1)(\delta_2+(\delta_1-\delta_2)\vol_c)$. 
\end{proof}

Therefore, with probability $1 - O(\frac{1}{n^2(\delta_2+(\delta_1-\delta_2)\vol_c)})$, the space complexity for $\anng$ $G_{\tau,\delta_1,\delta_2,P}$~(see  Definition~\ref{def:graph_oracle_twoside:formal}) is $O(nd+n^2(\delta_2+(\delta_1-\delta_2)\vol_c))$.

\subsection{Guarantees}\label{sec:sup_proof}

\subsubsection{Probability of Making Progress}
In this section, we present the lower bound for the probability of making a step progress towards the nearest neighbor. 
\begin{lemma}[]\label{lem:progress_twoside}
With parameters defined in Definition~\ref{def:oracle_param_twoside},
given a query vector $q\in \mathbb{S}^{d-1}$, we show that for any $p_1\in P$ such that $\sin \theta_{p_1,q}= (s+\varepsilon)2^{-\omega}$, if $\tau\geq \sqrt{2}(s+\epsilon)$, then
\begin{align*}
    \Pr\big[\exists p_{2}\in P, \ p_{2}\in{\cal N}_{\tau,\delta_1,\delta_2}(p_1) \ \textit{and} \ \sin{\theta_{p_2,q}}\leq s\cdot2^{-\omega}\big] \geq 1 - O(e^{-s^{d}\delta_1/\sqrt{d}}).
\end{align*}
\end{lemma}

\begin{proof}
We start with lower bounding the probability that a data point $p_2\in {\cal U}({\cal S}^{d-1})$ satisfies $\langle p_2, p_1 \rangle\geq \alpha_\tau$ and $\sin{\theta_{p_2,q}}\leq s\cdot2^{-\omega}$.

\begin{align}\label{eq:one_point_progress_twoside}
     \Pr\big[p_2\in {\cal U}({\cal S}^{d-1}), \ \langle p_{2},p_1 \rangle\geq \alpha_\tau , \sin{\theta_{p_2,q}}\leq s\cdot2^{-\omega}\big]
     & ~ = \vol_w(\alpha_\tau, \alpha_s, \arcsin((s+ \varepsilon)2^{-\omega}))\notag\\
     & ~ > \frac{s^{d}}{n\sqrt{d}},
\end{align}
where the first step follows from the definition of the volume of wedge in Definition~\ref{def:volume_wedge}, the second step follows from Lemma~\ref{lem:varepsilon}.

Next, we lower bound the probability that a data point $p_2\in {\cal U}({\cal S}^{d-1})$ satisfies $\langle p_2, p_1 \rangle\geq \alpha_\tau$ and $\sin{\theta_{p_2,q}}\leq s\cdot2^{-\omega}$ stays in the ${\cal N}_{\tau,\delta_1,\delta_2,P}(p_1)$.

\begin{align}\label{eq:one_point_noise_progress_twoside}
   & ~ \Pr\big[p_2\in {\cal U}({\cal S}^{d-1}), \ p_{2}\in{\cal N}_{\tau,\delta_1,\delta_2,P}(p_1) \ \textit{and}  \ \sin{\theta_{p_2,q}}\leq s\cdot2^{-\omega}\big] \notag \\
    = & ~  \delta_1 \Pr\big[p_2\in {\cal U}({\cal S}^{d-1}), \ \langle p_{2},p_1 \rangle\geq \alpha_\tau \ \textit{and} \ \sin{\theta_{p_2,q}}\leq s\cdot2^{-\omega}\big]  \notag\\
    > & ~   \frac{s^{d}\delta_1}{n\sqrt{d}},
\end{align}
where the first step follows from the definition of ${\cal N}_{\tau,\delta_1,\delta_2,P}(p_1)$ in Definition~\ref{def:oracle_param_twoside}, the second step follows from Eq.~\eqref{eq:one_point_progress_twoside}.

Next, we could lower bound the probability that there exists a $p_2\in P$ that satisfies $p_{2}\in{\cal N}_{\tau,\delta_1,\delta_2,P}(p_1)$ and $\sin{\theta_{p_2,q}}\leq s\cdot2^{-\omega}$ by union bounding the probability in Eq.~\eqref{eq:one_point_noise_progress_twoside}.

\begin{align*}
    \Pr\big[\exists p_{2}\in P, \ p_{2}\in{\cal N}_{\tau,\delta_1,\delta_2,P}(p_1) \ \textit{and} \ \sin{\theta_{p_2,q}}\leq s\cdot2^{-\omega}\big] 
     & ~ > 1 - (1 - \frac{s^{d}\delta_1}{n\sqrt{d}})^{n-1}  \\
     & ~ = 1 - O(e^{-s^{d}\delta_1/\sqrt{d}}),
\end{align*}
where the first step follow from Eq.~\eqref{eq:one_point_noise_progress_twoside}, the second step is a reorganization. 
\end{proof}

\subsubsection{Main Results}

\begin{theorem}[]\label{thm:main_twoside:formal}

Let $\omega>1$ and $r\in (1,2^{\omega})$ denote two parameters.  Let $r_0>r$ and  $\epsilon\in (0, r_0-r)$. Let $\tau\geq \sqrt{2}r_0$. Let $P \subset \mathbb{S}^{d-1}$ denote a $n$-point, $\omega$-dense dataset (see Definition~\ref{def:dense_dataset})  where every $p\in P$ is i.i.d. sampled from ${\cal U}(\mathbb{S}^{d-1})$. Let $\delta_1, \delta_2\in (0,1)$ and $\delta_1>\delta_2$.
Let ${\cal O}_{\tau,\delta_1,\delta_2,P}$ denote the oracle defined in Definition~\ref{def:graph_oracle_twoside:formal}. ${\cal O}_{\tau,\delta_1,\delta_2,P}$ is associated with a $\anng$ $G_{\tau,\delta_1,\delta_2,P}$~(see  Definition~\ref{def:graph_oracle_twoside:formal}). Let ${\cal N}_{\tau,\delta_1,\delta_2,P}(p)$ denote the neighbor set (see  Definition~\ref{def:graph_oracle_twoside:formal}) for vertex $p\in P$ on $G_{\tau,\delta_1,\delta_2,P}$~(see  Definition~\ref{def:graph_oracle_twoside:formal}). We define $N_{\tau,\delta_1,\delta_2,P}(p)$ as the size of the neighbor set ${\cal N}_{\tau,\delta_1,\delta_2,P}(p)$ for $p\in P$. We define $E_{\tau,\delta_1,\delta_2,P}$ as the number of edges on $G_{\tau,\delta_1,\delta_2,P}$.

Given a query $q\in\mathbb{S}^{d-1}$ with the premise that $\min_{p\in P}\sin{\theta_{p,q}}\leq r\cdot 2^{-\omega}$, there is an algorithm using ${\cal O}_{\tau,\delta_1,\delta_2,P}$ that takes ${\cal T}_{\mathsf{query}}$ query time, $O(n^2d\delta_1)$ preprocessing time and $O(n(\tau^d+d))$ space, starting from an $p_0\in P$ with $\sin \theta_{p_0,q}\leq r_0 2^{-\omega}$, that solves $r$-$\nns$ problem (Definition~\ref{def:nn:formal}) using $T$ steps 
with probability $1 - T \cdot e^{-r^{d}\delta_1/\sqrt{d}}$.
Moreover,
\begin{itemize}
    \item If $\delta_2=0$, ${\cal T}_{\mathsf{query}}=O(2^{\omega}\epsilon^{-1}d^{1/2}\tau^d\cdot\delta_1)$
    \item If $\delta_2=d^{-1/2}\tau^d2^{-d\omega}$,${\cal T}_{\mathsf{query}}=O(2^{\omega}\epsilon^{-1}d^{1/2}\tau^d)$.  
\end{itemize}
\end{theorem}

\begin{proof}
We start with showing how to make $\epsilon\cdot 2^{-\omega}$ progress from $p_0$. Let $s=r_0-\epsilon$ so that we have $\tau\geq\sqrt{2} (s+\epsilon)$. Then, we have
\begin{align}\label{eq:induct_first_twoside}
    \Pr\big[\exists p_{1}\in P, \ p_{1}\in{\cal N}_{\tau,\delta_1,\delta_2}(p_0) \  \textit{and} \ \sin{\theta_{p_1,q}} \leq (r_0-\epsilon)\cdot2^{-\omega}\big] 
     & \geq ~ 1 - O(e^{-(r_0-\epsilon)^{d}\delta_1/\sqrt{d}})\notag\\
     & > ~ 1-O(e^{-r^{d}\delta_1/\sqrt{d}}),
\end{align}
where the first step follows from Lemma~\ref{lem:progress_twoside}, the second step follows from $r_0-\epsilon>r$.

Then, applying induction over Eq.~\eqref{eq:induct_first_twoside}, we show that for $i\in \mathbb{N}$,
\begin{align}\label{eq:induct_twoside}
    \Pr\big[\exists p_{i+1}\in P, \  p_{i+1} \in{\cal N}_{\tau,\delta_1,\delta_2}(p_i) \ \textit{and} 
    \ \sin{\theta_{p_{i+1},q}}\leq (r_0-i\cdot \epsilon)\cdot2^{-\omega}\big] >  1 -O(e^{-r^{d}\delta_1/\sqrt{d}}).
\end{align}

Therefore, to reach a $p_T\in P$ such that $\sin{\theta_{p_T,q}}\leq r\cdot2^{-\omega}$, we should take $T=(r_0-r)2^{\omega}/\epsilon=O(2^{\omega}/\epsilon)$ steps. Moreover, we could union bound the total failure probability as $T \cdot e^{-r^{d}\delta_1/\sqrt{d}} $. 

Next, we analyze the complexity for query time, preprocessing time and space.
{\bf Preprocessing time:}
Build an oracle $O_{\tau,\delta,P}$ (see Definition~\ref{def:graph_oracle}) requires compute pairwise distances for vectors $p\in P$. As $\anng$ only samples a subset of $P$ with size $O(n\delta_1)$ for any $p\in P$ to compute, the preprocessing time is $O(n^2d\delta_1)$.

{\bf Query time:}
From Lemma~\ref{lem:1-steps_twoside}, we know that for a query vector $q\in \mathbb{S}^{d-1}$, with probability $1 - O(\frac{1}{ n(\delta_2+(\delta_1-\delta_2)\vol_c)})$, the time complexity of a greedy step (Definition~\ref{def:greedy_step_twoside}) for $q$ at $p$ using oracle ${\cal O}_{\tau,\delta_1,\delta_2,P}$ (see Definition~\ref{def:graph_oracle_twoside:formal}) is $ (n-1)d(\delta_2+(\delta_1-\delta_2)\vol_c)$. Following this lemma, we write the total query time is
\begin{align*}
    T(n-1)(n-1)d(\delta_2+(\delta_1-\delta_2)\vol_c) 
       =  O(2^{\omega}\epsilon^{-1}\cdot (n-1)d(\delta_2+(\delta_1-\delta_2)d^{-1/2} \tau^d2^{-d\omega}) ),
\end{align*}
where the first step follows Lemma~\ref{lem:prep_vol_wedge}.

Thus, we have
\begin{itemize}
    \item If $\delta_2=0$, ${\cal T}_{\mathsf{query}}=O(2^{\omega}\epsilon^{-1}d^{1/2}\tau^d\cdot\delta_1)$
    \item If $\delta_2=d^{-1/2}\tau^d2^{-d\omega}$,${\cal T}_{\mathsf{query}}=O(2^{\omega}\epsilon^{-1}d^{1/2}\tau^d)$.  
\end{itemize}

Next, we union bound the failure probability of achieving the query time complexity as 
\begin{align*}
    O(\frac{T}{ n(\delta_2+(\delta_1-\delta_2)\vol_c)})
    =  O(\frac{2^{\omega}\epsilon^{-1}}{ n\cdot (\delta_2+(\delta_1-\delta_2)d^{-1/2} \tau^d2^{-d\omega})}),
\end{align*}
where the first step follows Lemma~\ref{lem:prep_vol_wedge}.

Thus, we have
\begin{itemize}
    \item If $\delta_2=0$, the failure probability is $O(2^{\omega}\epsilon^{-1}d^{1/2}\tau^{-d}\delta_1)$.
    \item If $\delta_2=d^{-1/2}\tau^d2^{-d\omega}$, the failure probability is $O(2^{\omega}\epsilon^{-1}d^{1/2}\tau^{-d})$.
\end{itemize}

{\bf  Space:}
We need to store all the vectors in $P$, which takes $O(nd)$ space. Moreover, following Lemma~\ref{lem:space}, the number of edges for $G_{\tau,\delta_1,\delta_2,P}$ used in oracle $O_{\tau,\delta_1,\delta_2,P}$ (see Definition~\ref{def:graph_oracle_twoside:formal}) is:
\begin{align*}
    E_{\tau,\delta_1,\delta_2,P} 
     = O( n^2(\delta_2+(\delta_1-\delta_2)\vol_c))
   = O(n d^{-1/2} \tau^d),
\end{align*}
where the first step follows Lemma~\ref{lem:prep_vol_wedge},the second step follows from $2^{-d\omega}=\frac{1}{n}$ in Fact~\ref{fact:omega_to_d}.

Next, we write the failure probability as $ O(\frac{1}{n^2(\delta_2+(\delta_1-\delta_2)\vol_c)})=O(\frac{1}{nd^{-1/2}  \tau^d})$.

Therefore, with probability $1-O(\frac{1}{nd^{-1/2} \tau^d})$, we show that the space complexity is $O(nd^{-1/2}\tau^d+nd)=O(n(\tau^d+d))$. 
\end{proof}

From Theorem~\ref{thm:main_twoside:formal}, we have two major takeaways: (1) if we perform random sampling when we preprocess the dataset $P$ with sample probability $\delta_1$, we could save the query time with a constant factor $\Theta(\delta_1)$, but failure probability would be raised to the power of $\Theta(\delta)$, (2) if we add random edges to improve the connectivity of the graph, with careful choice of random sample probability $\delta_2$, we could maintain the samee query time as the greedy search on the exact nearest neighbor graph. 

\section{Conclusion}
In this paper, we bridge a practice-to-theory gap in graph-based nearest neighbor search~($\nns$) algorithms. While current theoretical research focuses on analyzing a greedy search on an exact near neighbor graph, we propose the theoretical guarantees for a greedy search on the approximate nearest neighbor graph ($\anng$).  Our analysis quantifies the trade-offs between the approximation quality of $\anng$ and the search efficiency for the greedy style query algorithm. Our theoretical guarantee indicates that an approximation to the exact near neighbor graph would reduce the query time but increase the failure probability in solving $\nns$. We hope our analysis will shed practical light on the design of new graph-based $\nns$ algorithms.

\bibliographystyle{plainnat}
\bibliography{ref}

\begin{thebibliography}{34}
\providecommand{\natexlab}[1]{#1}
\providecommand{\url}[1]{\texttt{#1}}
\expandafter\ifx\csname urlstyle\endcsname\relax
  \providecommand{\doi}[1]{doi: #1}\else
  \providecommand{\doi}{doi: \begingroup \urlstyle{rm}\Url}\fi

\bibitem[Altman(1992)]{a92}
Naomi~S Altman.
\newblock An introduction to kernel and nearest-neighbor nonparametric
  regression.
\newblock \emph{The American Statistician}, 46\penalty0 (3):\penalty0 175--185,
  1992.

\bibitem[Amagata and Hara(2021)]{ah21}
Daichi Amagata and Takahiro Hara.
\newblock Reverse maximum inner product search: How to efficiently find users
  who would like to buy my item?
\newblock In \emph{Fifteenth ACM Conference on Recommender Systems}, pages
  273--281, 2021.

\bibitem[Andoni and Razenshteyn(2015)]{ar15}
Alexandr Andoni and Ilya Razenshteyn.
\newblock Optimal data-dependent hashing for approximate near neighbors.
\newblock In \emph{Proceedings of the forty-seventh annual ACM symposium on
  Theory of computing (STOC)}, pages 793--801, 2015.

\bibitem[Andoni et~al.(2017)Andoni, Laarhoven, Razenshteyn, and
  Waingarten]{alrw17}
Alexandr Andoni, Thijs Laarhoven, Ilya Razenshteyn, and Erik Waingarten.
\newblock Optimal hashing-based time-space trade-offs for approximate near
  neighbors.
\newblock In \emph{Proceedings of the Twenty-Eighth Annual ACM-SIAM Symposium
  on Discrete Algorithms (SODA)}, pages 47--66. SIAM, 2017.

\bibitem[Aum{\"u}ller et~al.(2017)Aum{\"u}ller, Bernhardsson, and
  Faithfull]{abf17}
Martin Aum{\"u}ller, Erik Bernhardsson, and Alexander Faithfull.
\newblock Ann-benchmarks: A benchmarking tool for approximate nearest neighbor
  algorithms.
\newblock In \emph{International Conference on Similarity Search and
  Applications}, pages 34--49. Springer, 2017.

\bibitem[Baranchuk et~al.(2018)Baranchuk, Babenko, and Malkov]{bbm18}
Dmitry Baranchuk, Artem Babenko, and Yury Malkov.
\newblock Revisiting the inverted indices for billion-scale approximate nearest
  neighbors.
\newblock In \emph{Proceedings of the European Conference on Computer Vision
  (ECCV)}, pages 202--216, 2018.

\bibitem[Charikar(2002)]{c02}
Moses~S Charikar.
\newblock Similarity estimation techniques from rounding algorithms.
\newblock In \emph{Proceedings of the thiry-fourth annual ACM symposium on
  Theory of computing}, pages 380--388, 2002.

\bibitem[Chen et~al.(2020{\natexlab{a}})Chen, Liu, Peng, Xu, Li, Dao, Song,
  Shrivastava, and Re]{clp+20}
Beidi Chen, Zichang Liu, Binghui Peng, Zhaozhuo Xu, Jonathan~Lingjie Li, Tri
  Dao, Zhao Song, Anshumali Shrivastava, and Christopher Re.
\newblock Mongoose: A learnable lsh framework for efficient neural network
  training.
\newblock In \emph{International Conference on Learning Representations},
  2020{\natexlab{a}}.

\bibitem[Chen et~al.(2020{\natexlab{b}})Chen, Medini, Farwell, Tai,
  Shrivastava, et~al.]{cmf+20}
Beidi Chen, Tharun Medini, James Farwell, Charlie Tai, Anshumali Shrivastava,
  et~al.
\newblock Slide: In defense of smart algorithms over hardware acceleration for
  large-scale deep learning systems.
\newblock \emph{Proceedings of Machine Learning and Systems}, 2:\penalty0
  291--306, 2020{\natexlab{b}}.

\bibitem[Chen et~al.(2009)Chen, Fang, and Saad]{cfs09}
Jie Chen, Haw-ren Fang, and Yousef Saad.
\newblock Fast approximate knn graph construction for high dimensional data via
  recursive lanczos bisection.
\newblock \emph{Journal of Machine Learning Research}, 10\penalty0 (9), 2009.

\bibitem[Datar et~al.(2004)Datar, Immorlica, Indyk, and Mirrokni]{dii+04}
Mayur Datar, Nicole Immorlica, Piotr Indyk, and Vahab~S Mirrokni.
\newblock Locality-sensitive hashing scheme based on p-stable distributions.
\newblock In \emph{Proceedings of the twentieth annual symposium on
  Computational geometry}, pages 253--262, 2004.

\bibitem[Dong et~al.(2011)Dong, Moses, and Li]{dml11}
Wei Dong, Charikar Moses, and Kai Li.
\newblock Efficient k-nearest neighbor graph construction for generic
  similarity measures.
\newblock In \emph{Proceedings of the 20th international conference on World
  wide web}, pages 577--586, 2011.

\bibitem[Fan et~al.(2019)Fan, Guo, Zhu, Miao, Sun, and Li]{fgz+19}
Miao Fan, Jiacheng Guo, Shuai Zhu, Shuo Miao, Mingming Sun, and Ping Li.
\newblock Mobius: towards the next generation of query-ad matching in baidu's
  sponsored search.
\newblock In \emph{Proceedings of the 25th ACM SIGKDD International Conference
  on Knowledge Discovery \& Data Mining}, pages 2509--2517, 2019.

\bibitem[Fix and Hodges(1989)]{fh89}
Evelyn Fix and Joseph~Lawson Hodges.
\newblock Discriminatory analysis. nonparametric discrimination: Consistency
  properties.
\newblock \emph{International Statistical Review/Revue Internationale de
  Statistique}, 57\penalty0 (3):\penalty0 238--247, 1989.

\bibitem[Fu et~al.(2017)Fu, Xiang, Wang, and Cai]{fxw+17}
Cong Fu, Chao Xiang, Changxu Wang, and Deng Cai.
\newblock Fast approximate nearest neighbor search with the navigating
  spreading-out graph.
\newblock \emph{arXiv preprint arXiv:1707.00143}, 2017.

\bibitem[Hajebi et~al.(2011)Hajebi, Abbasi-Yadkori, Shahbazi, and
  Zhang]{has+11}
Kiana Hajebi, Yasin Abbasi-Yadkori, Hossein Shahbazi, and Hong Zhang.
\newblock Fast approximate nearest-neighbor search with k-nearest neighbor
  graph.
\newblock In \emph{Twenty-Second International Joint Conference on Artificial
  Intelligence}, 2011.

\bibitem[Jaffe et~al.(2020)Jaffe, Kluger, Linderman, Mishne, and
  Steinerberger]{jkl+20}
Ariel Jaffe, Yuval Kluger, George~C Linderman, Gal Mishne, and Stefan
  Steinerberger.
\newblock Randomized near-neighbor graphs, giant components and applications in
  data science.
\newblock \emph{Journal of applied probability}, 57\penalty0 (2):\penalty0
  458--476, 2020.

\bibitem[Jegou et~al.(2010)Jegou, Douze, and Schmid]{jds10}
Herve Jegou, Matthijs Douze, and Cordelia Schmid.
\newblock Product quantization for nearest neighbor search.
\newblock \emph{IEEE transactions on pattern analysis and machine
  intelligence}, 33\penalty0 (1):\penalty0 117--128, 2010.

\bibitem[Khandelwal et~al.(2019)Khandelwal, Levy, Jurafsky, Zettlemoyer, and
  Lewis]{klj+19}
Urvashi Khandelwal, Omer Levy, Dan Jurafsky, Luke Zettlemoyer, and Mike Lewis.
\newblock Generalization through memorization: Nearest neighbor language
  models.
\newblock In \emph{International Conference on Learning Representations}, 2019.

\bibitem[Laarhoven(2018)]{l18}
Thijs Laarhoven.
\newblock Graph-based time-space trade-offs for approximate near neighbors.
\newblock In \emph{34th International Symposium on Computational Geometry (SoCG
  2018)}, pages 1--14. Schloss Dagstuhl-Leibniz-Zentrum f{\"u}r Informatik,
  2018.

\bibitem[Makarychev et~al.(2020)Makarychev, Reddy, and Shan]{mrs20}
Konstantin Makarychev, Aravind Reddy, and Liren Shan.
\newblock Improved guarantees for k-means++ and k-means++ parallel.
\newblock \emph{Advances in Neural Information Processing Systems}, 33, 2020.

\bibitem[Malkov and Yashunin(2018)]{my18}
Yu~A Malkov and Dmitry~A Yashunin.
\newblock Efficient and robust approximate nearest neighbor search using
  hierarchical navigable small world graphs.
\newblock \emph{IEEE transactions on pattern analysis and machine
  intelligence}, 42\penalty0 (4):\penalty0 824--836, 2018.

\bibitem[Newman et~al.(2006)Newman, Barab{\'a}si, and Watts]{nb06}
Mark~Ed Newman, Albert-L{\'a}szl{\'o}~Ed Barab{\'a}si, and Duncan~J Watts.
\newblock \emph{The structure and dynamics of networks.}
\newblock Princeton university press, 2006.

\bibitem[Prokhorenkova and Shekhovtsov(2020)]{ps20}
Liudmila Prokhorenkova and Aleksandr Shekhovtsov.
\newblock Graph-based nearest neighbor search: From practice to theory.
\newblock In \emph{International Conference on Machine Learning}, pages
  7803--7813. PMLR, 2020.

\bibitem[Ramasubramanian and Paliwal(1992)]{rp92}
V~Ramasubramanian and Kuldip~K Paliwal.
\newblock Fast k-dimensional tree algorithms for nearest neighbor search with
  application to vector quantization encoding.
\newblock \emph{IEEE Transactions on Signal Processing}, 40\penalty0
  (3):\penalty0 518--531, 1992.

\bibitem[Sarwar et~al.(2001)Sarwar, Karypis, Konstan, and Riedl]{skk+01}
Badrul Sarwar, George Karypis, Joseph Konstan, and John Riedl.
\newblock Item-based collaborative filtering recommendation algorithms.
\newblock In \emph{Proceedings of the 10th international conference on World
  Wide Web}, pages 285--295, 2001.

\bibitem[Shakhnarovich et~al.(2006)Shakhnarovich, Darrell, and Indyk]{sdi06}
Gregory Shakhnarovich, Trevor Darrell, and Piotr Indyk.
\newblock \emph{Nearest-neighbor methods in learning and vision: theory and
  practice (neural information processing)}.
\newblock The MIT press, 2006.

\bibitem[Song and Yu(2021)]{sy21}
Zhao Song and Zheng Yu.
\newblock Oblivious sketching-based central path method for solving linear
  programming problems.
\newblock In \emph{38th International Conference on Machine Learning (ICML)},
  2021.

\bibitem[Song et~al.(2021)Song, Woodruff, Yu, and Zhang]{swy+21}
Zhao Song, David Woodruff, Zheng Yu, and Lichen Zhang.
\newblock Fast sketching of polynomial kernels of polynomial degree.
\newblock In \emph{International Conference on Machine Learning}, pages
  9812--9823. PMLR, 2021.

\bibitem[Tan et~al.(2019)Tan, Zhou, Xu, and Li]{tzx+19}
Shulong Tan, Zhixin Zhou, Zhaozhuo Xu, and Ping Li.
\newblock On efficient retrieval of top similarity vectors.
\newblock In \emph{Proceedings of the 2019 Conference on Empirical Methods in
  Natural Language Processing and the 9th International Joint Conference on
  Natural Language Processing (EMNLP-IJCNLP)}, pages 5236--5246, 2019.

\bibitem[Tan et~al.(2021)Tan, Xu, Zhao, Fei, Zhou, and Li]{txz+21}
Shulong Tan, Zhaozhuo Xu, Weijie Zhao, Hongliang Fei, Zhixin Zhou, and Ping Li.
\newblock Norm adjusted proximity graph for fast inner product retrieval.
\newblock In \emph{Proceedings of the 27th ACM SIGKDD Conference on Knowledge
  Discovery \& Data Mining}, pages 1552--1560, 2021.

\bibitem[Zhang et~al.(2013)Zhang, Huang, Geng, and Liu]{zhg+13}
Yan-Ming Zhang, Kaizhu Huang, Guanggang Geng, and Cheng-Lin Liu.
\newblock Fast knn graph construction with locality sensitive hashing.
\newblock In \emph{Joint European Conference on Machine Learning and Knowledge
  Discovery in Databases}, pages 660--674. Springer, 2013.

\bibitem[Zhao et~al.(2020)Zhao, Tan, and Li]{ztl20}
Weijie Zhao, Shulong Tan, and Ping Li.
\newblock Song: Approximate nearest neighbor search on gpu.
\newblock In \emph{2020 IEEE 36th International Conference on Data Engineering
  (ICDE)}, pages 1033--1044. IEEE, 2020.

\bibitem[Zhou et~al.(2019)Zhou, Tan, Xu, and Li]{ztx+19}
Zhixin Zhou, Shulong Tan, Zhaozhuo Xu, and Ping Li.
\newblock M{\"o}bius transformation for fast inner product search on graph.
\newblock In \emph{Proceedings of the 33rd International Conference on Neural
  Information Processing Systems}, pages 8218--8229, 2019.

\end{thebibliography}
\end{document}